\documentclass{article}
\usepackage{amsmath,amsfonts,bm}

\def\eqref#1{equation~\ref{#1}}

\def\1{\bm{1}}

\DeclareMathAlphabet{\mathsfit}{\encodingdefault}{\sfdefault}{m}{sl}
\SetMathAlphabet{\mathsfit}{bold}{\encodingdefault}{\sfdefault}{bx}{n}

\newcommand{\pdata}{p_{\rm{data}}}

\newcommand{\punbias}{p_{\rm{ref}}}
\newcommand{\pbias}{p_{\rm{bias}}}

\newcommand{\Dbias}{{\mathcal{D}_{\mathrm{bias}}}}
\newcommand{\Dunbias}{{\mathcal{D}_{\mathrm{ref}}}}

\newcommand{\cX}{{\mathcal{X}}}

\newcommand{\cD}{{\mathcal{D}}}

\newcommand{\cL}{{\mathcal{L}}}
\newcommand{\bbR}{{\mathbb{R}}}
\newcommand{\bbE}{{\mathbb{E}}}
\newcommand{\bx}{{\mathbf{x}}}
\newcommand{\bu}{{\mathbf{u}}}
\newcommand{\bz}{{\mathbf{z}}}
\usepackage{natbib}
\usepackage{algorithm}
\usepackage{algorithmic}

\usepackage{amsthm}
\usepackage{microtype}
\usepackage{graphicx}
\usepackage{subfigure}
\usepackage{booktabs} %
\usepackage{url}
\usepackage{booktabs}
\usepackage{siunitx}
\usepackage{nicefrac}
\usepackage{wrapfig}
\usepackage{cancel}

\newtheorem{theorem}{Theorem}

\usepackage{hyperref}

\usepackage[accepted]{icml2020}

\icmltitlerunning{Fair Generative Modeling via Weak Supervision}

\begin{document}

\twocolumn[
\icmltitle{Fair Generative Modeling via Weak Supervision}

\icmlsetsymbol{equal}{*}

\begin{icmlauthorlist}
\icmlauthor{Kristy Choi}{stan,equal}
\icmlauthor{Aditya Grover}{stan,equal}
\icmlauthor{Trisha Singh}{stanstats}
\icmlauthor{Rui Shu}{stan}
\icmlauthor{Stefano Ermon}{stan}
\end{icmlauthorlist}

\icmlaffiliation{stan}{Department of Computer Science, Stanford University}
\icmlaffiliation{stanstats}{Department of Statistics, Stanford University}

\icmlcorrespondingauthor{Kristy Choi}{kechoi@cs.stanford.edu}

\icmlkeywords{Machine Learning, ICML}

\vskip 0.3in
]

\printAffiliationsAndNotice{\icmlEqualContribution} %

\begin{abstract}
Real-world datasets are often biased with respect to key demographic factors such as race and gender.
Due to the latent nature of the underlying factors, detecting and mitigating bias is especially challenging for unsupervised machine learning.
We present a weakly supervised algorithm for overcoming dataset bias for deep generative models.
Our approach requires access to an additional small, unlabeled 
reference dataset as the supervision signal, thus sidestepping the need for explicit labels on the underlying bias factors. 
Using this supplementary dataset, we detect the bias in existing datasets via a density ratio technique and learn generative models which efficiently achieve the twin goals of: 1) data efficiency by using training examples from both biased and reference datasets for learning; and
2) data generation close in distribution to the reference dataset at test time.
Empirically, we demonstrate the efficacy of our approach which reduces bias w.r.t. latent factors by an average of up to 34.6\% over baselines for comparable image generation using generative adversarial networks.
\end{abstract}

\section{Introduction}
Increasingly, many applications of machine learning (ML) involve \textit{data generation}. 
Examples of such production level systems include Transformer-based models such as BERT and GPT-3 for natural language generation~\citep{vaswani2017attention,devlin2018bert,radford2019language,brown2020language}, Wavenet for text-to-speech synthesis~\citep{oord2017parallel}, and a large number of creative applications such Coconet used for designing the ``first AI-powered Google Doodle''~\citep{huang2017counterpoint}.
As these generative applications 
become more prevalent, it becomes increasingly important to consider
questions with regards to the potential discriminatory nature of such systems and ways to mitigate it~\citep{podesta2014big}. For example, some natural language generation systems trained on internet-scale datasets have been shown to produce generations that are biased towards certain demographics~\citep{sheng2019woman}.

A variety of socio-technical factors contribute to the discriminatory nature of ML systems~\citep{barocas-hardt-narayanan}.
A major factor is the existence of biases in the training data itself~\citep{torralba2011unbiased, tommasi2017deeper}.
Since data is the fuel of ML, any existing bias in the dataset can be propagated to the learned model~\citep{barocas2016big}.
This is a particularly pressing concern for generative models which can easily amplify the bias by generating more of the biased data at test time.
Further, learning a generative model is fundamentally an unsupervised learning problem and hence, the bias factors of interest are typically latent.
For example, while learning a generative model of human faces, we often do not have access to attributes such as gender, race, and age.
Any existing bias in the dataset with respect to these attributes are easily picked by deep generative models. See Figure~\ref{fig:baseline} for an illustration.

\begin{figure*}[ht]
    \centering
    \includegraphics[width=.9\linewidth]{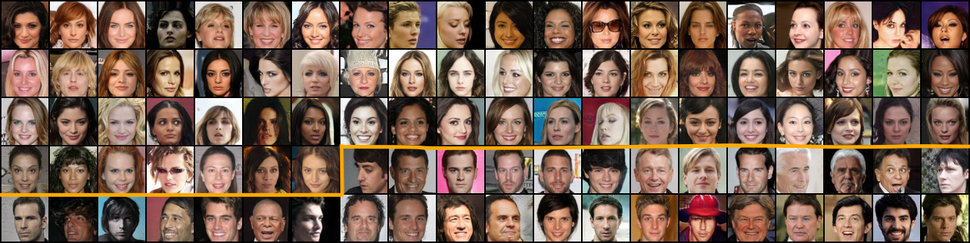}
    \caption{Samples from a baseline BigGAN that reflect the gender bias underlying the true data distribution in CelebA. All faces above the orange line (67\%) are classified as female, while the rest are labeled as male (33\%).}
    \label{fig:baseline}
\end{figure*}

In this work, we present a weakly-supervised approach to learning \textit{fair} generative models in the presence of dataset bias.
Our source of weak supervision is motivated by the observation that obtaining multiple unlabelled (biased) datasets is relatively cheap for many domains in the big data era. Among these data sources, we may wish to generate samples that are close in distribution to a particular target (reference) dataset.\footnote{We note that while there may not be concept of a dataset devoid of bias, carefully designed representative data collection practices may be more accurately reflected in some data sources~\citep{gebru2018datasheets} and can be considered as \textit{reference} datasets.}
As a concrete example of such a reference, organizations such as the World Bank and biotech firms~\citep{23andme,hong76587923andme} typically follow several good practices to ensure representativeness in the datasets that they collect, though such methods are unscalable to large sizes.
We note that neither of our datasets need to be labeled w.r.t. the latent bias attributes and the size of the reference dataset can be much smaller than the biased dataset.
Hence, the level of supervision we require is weak.

Using a reference dataset to augment a  biased dataset, our goal is to learn a generative model that best approximates the desired, reference data distribution.
Simply using the reference dataset alone for learning is an option, but this may not suffice since this dataset can be too small to learn an expressive model that accurately captures the underlying reference data distribution.
Our approach to learning a \textit{fair} generative model that is robust to biases in the larger training set is based on importance reweighting.
In particular, we learn a generative model which reweighs the data points in the biased dataset based on the ratio of densities assigned by the biased data distribution as compared to the reference data distribution.
Since we do not have access to explicit densities assigned by either of the two distributions, we estimate the weights by using a probabilistic classifier~\citep{sugiyama2012density,mohamed2016learning}.

We test our weakly-supervised approach on learning generative adversarial networks on the CelebA dataset~\citep{liu2015faceattributes}.
The dataset consists of attributes such as gender and hair color, which we use for designing biased and reference data splits and subsequent evaluation. We empirically demonstrate how the reweighting approach can offset dataset bias on a wide range of settings.
In particular, we obtain improvements of up to 36.6\% (49.3\% for \texttt{bias}=0.9 and 23.9\% for \texttt{bias}=0.8) for single-attribute dataset bias and 32.5\% for multi-attribute dataset bias on average over baselines in reducing the bias with respect to the latent factors for comparable sample quality.

\section{Problem Setup}
\subsection{Background}
We assume there exists a true (unknown) data distribution $\pdata: \cX \to  \bbR_{\ge 0}$ over a set of $d$ observed variables $\bx \in \bbR^d$.
In generative modeling, our goal is to learn the parameters $\theta\in\Theta$ of a distribution $p_\theta: \cX \to \bbR_{\ge 0}$ over the observed variables $\bx$, such that the model distribution $p_\theta$ is close to $\pdata$.
Depending on the choice of learning algorithm, different approaches have been previously considered.
Broadly, these include adversarial training e.g., GANs~\citep{goodfellow2014generative} and maximum likelihood estimation (MLE) e.g., variational autoencoders (VAE)~\citep{kingma2013auto,rezende2014stochastic} and normalizing flows~\citep{dinh2014nice} or hybrids~\citep{grover2018flow}.
Our bias mitigation framework is agnostic to the above training approaches.

For generality, we consider expectation-based learning objectives, where $\ell(\cdot)$ is a per-example loss that depends on both examples $\bx$ drawn from a dataset $\cD$ and the model parameters $\theta$:
\begin{align}\label{eq:gen_model_learn}
    \bbE_{\bx \sim \pdata}[\ell(\bx, \theta)] \approx \frac{1}{T} \sum_{i=1}^T \ell(\bx_i, \theta):= \cL(\theta; \cD) 
\end{align}

The above expression encompasses a broad class of MLE and adversarial objectives. For example, if $\ell(\cdot)$ denotes the negative log-likelihood assigned to the point $\bx$ as per $p_\theta$, then we recover the MLE training objective. 

\subsection{Dataset Bias} 
The standard assumption for learning a generative model is that we have access to a sufficiently large dataset $\Dunbias$ of training examples, where each $\bx\in \Dunbias$ is assumed to be sampled independently from a reference distribution $\pdata=\punbias$.
In practice however, collecting large datasets that are i.i.d. w.r.t. $\punbias$ is difficult due to a variety of socio-technical factors. The sample complexity for learning high dimensional distributions can even be doubly-exponential in the dimensions in many cases~\citep{arora2018gans}, surpassing the size of the largest available datasets.

We can partially offset this difficulty by considering data from alternate sources related to the target distribution, e.g., images scraped from the Internet.
However, these additional datapoints are not expected to be i.i.d. w.r.t. $\punbias$.

We characterize this phenomena as \textit{dataset bias}, where we assume the availability of a dataset $\Dbias$, such that the examples $\bx\in \Dbias$ are sampled independently from a biased (unknown) distribution $\pbias$ that is different from $\punbias$, but shares the same support.

\subsection{Evaluation}\label{sec:evaluation_metrics}
Evaluating generative models and fairness in machine learning are both open areas of research. Our work is at the intersection of these two fields and we propose the following metrics for measuring bias mitigation for data generation.
\paragraph{Sample Quality:} We employ sample quality metrics e.g., Frechet Inception Distance (FID)~\citep{heusel2017gans}, Kernel Inception Distance (KID)~\citep{li2017mmd}, etc. These metrics match empirical expectations w.r.t. a reference data distribution $p$ and a model distribution $p_\theta$ in a predefined feature space e.g., the prefinal layer of activations of Inception Network~\citep{szegedy2016rethinking}.
A lower score indicates that the learned model can better approximate $\pdata$.
For the fairness context in particular, we are interested in measuring the discrepancy w.r.t. $\punbias$  even if the model has been trained to use both $\Dunbias$ and $\Dbias$. We refer the reader to Supplement~\ref{sec:unbiased_fid} for more details on evaluation with FID.
\paragraph{Fairness:} Alternatively, we can evaluate bias of generative models specifically in the context of some sensitive latent variables, say $\bu \in \bbR^k$. For example, $\bu$ may correspond to the age and gender of an individual depicted via an image $\bx$. We emphasize that such attributes are \textit{unknown} during training, and used only for evaluation at test time.
    
If we have access to a highly accurate predictor $p(\bu \vert \bx)$ for the distribution of the sensitive attributes $\bu$ conditioned on the observed $\bx$, we can evaluate the extent of bias mitigation via the discrepancies in the expected marginal likelihoods of $\bu$ as per $\punbias$ and $p_\theta$. 
    
    Formally, we define the fairness discrepancy $f$ for a generative model $p_\theta$ w.r.t. $\punbias$ and sensitive attributes $\bu$:
    \begin{align}\label{eq:fairness_discrepancy}
        f(\punbias, p_\theta) = \vert \bbE_{\punbias}[p(\bu \vert \bx)] - \bbE_{p_\theta}[p(\bu \vert \bx)] \vert_2.
    \end{align}
    In practice, the expectations in Eq.~\eqref{eq:fairness_discrepancy} can be computed via Monte Carlo averaging. Again the lower is the discrepancy in the above two expectations, the better is the learned model's ability to mitigate dataset bias w.r.t. the sensitive attributes $\bu$. We refer the reader to Supplement~\ref{sec:fairdisc_metric} for more details on the fairness discrepancy metric.

\section{Bias Mitigation}
\label{sec:iw}

We assume a learning setting where we are given access to a data source $\Dbias$ in addition to a dataset of training examples $\Dunbias$.
Our goal is to capitalize on both data sources $\Dbias$ and $\Dunbias$ for learning a model $p_\theta$ that best approximates the target distribution $\punbias$. 

\subsection{Baselines}\label{sec:baselines}
We begin by discussing two baseline approaches at the extreme ends of the spectrum.
First, one could completely ignore $\Dbias$ and consider learning $p_\theta$ based on $\Dunbias$ alone. Since we only consider proper losses w.r.t. $\punbias$, global optimization of the objective in Eq.~\eqref{eq:gen_model_learn} in a well-specified model family will recover the true data distribution as $\vert\Dunbias\vert \to \infty$. However, since $\Dunbias$ is finite in practice, this is likely to give poor sample quality even though the fairness discrepancy would be low.

On the other extreme, we can consider learning $p_\theta$ based on the full dataset consisting of both $\Dunbias$ and $\Dbias$.
This procedure will be data efficient and could lead to high sample quality, but it comes at the cost of fairness since the learned distribution will be heavily biased w.r.t. $\punbias$.

\subsection{Solution 1: Conditional Modeling}

Our first proposal 
is to learn a generative model \textit{conditioned} on the identity of the dataset used during training.
Formally, we learn a generative model $p_\theta (\bx \vert y)$
where $y\in\{0, 1\}$ is a binary random variable indicating whether the model distribution was learned to approximate the data distribution corresponding to $\Dunbias$ (i.e., $\punbias$) or $\Dbias$ (i.e., $\pbias$).
By sharing model parameters $\theta$ across the two values of $y$, we hope to leverage both data sources.
At test time, conditioning on $y$ for $\Dunbias$ should result in fair generations.

As we demonstrate in Section~\ref{sec:exp} however, this simple approach does not achieve the intended effect in practice.
The likely cause is that the conditioning information is too weak for the model to infer the bias factors and effectively distinguish between the two distributions.
Next, we present an alternate two-phased approach based on density ratio estimation which effectively overcomes the dataset bias in a data-efficient manner.

\subsection{ Solution 2: Importance Reweighting}

Recall a trivial baseline in Section~\ref{sec:baselines} which learns a generative model on the union of $\Dbias$ and $\Dunbias$. This method is problematic because it assigns equal weight to the loss contributions from each individual datapoint in our dataset in Eq.~\eqref{eq:gen_model_learn}, regardless of whether the datapoint comes from $\Dbias$ or $\Dunbias$.
For example, in situations where the dataset bias causes a minority group to be underrepresented, this objective will encourage the model to focus on the majority group such that the overall value of the loss is minimized on average with respect to a biased empirical distribution i.e., a weighted mixture of $\pbias$ and $\punbias$ with weights proportional to $\vert \Dbias \vert$ and $\vert\Dunbias\vert$.

Our key idea is to reweight the datapoints from $\Dbias$ during training such that the model learns to downweight over-represented data points from $\Dbias$ while simultaneously upweighting the under-represented points from $\Dunbias$.
The challenge in the unsupervised context is that we do not have direct supervision on which points are over- or under-represented and by how much.
To resolve this issue, we consider importance sampling~\citep{HorvitzTh52}.
Whenever we are given data from two distributions, w.l.o.g. say $p$ and $q$, and wish to evaluate a sample average w.r.t. $p$ given samples from $q$, we can do so by reweighting the samples from $p$ by the ratio of densities assigned to the sampled points by $p$ and $q$. 
In our setting, the distributions of interest are $\pbias$ and $\punbias$ respectively. 
Hence, an importance weighted objective for learning 
from 
$\Dbias$ is:
\begin{align}
     \bbE_{\bx \sim \punbias}[\ell(\bx, \theta)] &= \bbE_{\bx \sim \pbias}\left[\frac{\punbias(\bx)}{\pbias(\bx)}\ell(\bx, \theta)\right]\\
    &\approx \frac{1}{T} \sum_{i=1}^T w(\bx_i)\ell(\bx_i, \theta):=\cL(\theta, \Dbias)
\end{align}
where $w(\bx_i):=\frac{\punbias(\bx)}{\pbias(\bx)} $ is defined to be the importance weight for $\bx_i\sim \pbias$.

\paragraph{Estimating density ratios via binary classification.} 
To estimate the importance weights, we use a binary classifier as described below~\citep{sugiyama2012density}.

Consider a binary classification problem with classes $Y \in \{0,1\}$ with training data generated as follows. 
First, we fix a prior probability for $p(Y=1)$. 
Then, we repeatedly sample $y \sim p(Y)$. 
If $y=1$, we independently sample a datapoint $\bx \sim \punbias$, else we sample $\bx \sim \pbias$.
Then, as shown in~\citet{friedman2001elements}, the ratio of densities $\punbias$ and $\pbias$ assigned to an arbitrary point $\bx$ can be recovered via a Bayes optimal (probabilistic) classifier $c^\ast: \cX \to [0,1]$:
\begin{align}\label{eq:imp_wts}
    w(\bx) = \frac{\punbias(\bx)}{\pbias(\bx)} = \gamma \frac{c^\ast(Y=1 \vert x)}{1-c^\ast(Y=1 \vert x)}
\end{align}
where $c(Y=1 \vert x)$ is the probability assigned by the classifier to the point $\bx$ belonging to class $Y=1$. Here, $\gamma=\frac{p(Y=0)}{p(Y=1)}$
is the ratio of marginals of the labels for two classes. 

In practice, we do not have direct access to either $\pbias$ or $\punbias$ and hence, our training data consists of points sampled from the empirical data distributions defined uniformly over $\Dunbias$ and $\Dbias$.
Further, we may not be able to learn a Bayes optimal classifier and denote the importance weights estimated by the learned classifier $c$ for a point $\bx$ as $\hat{w}(\bx)$.

\begin{algorithm}[t]
  \caption{Learning Fair Generative Models}
  \label{alg:fair_gen}
  \textbf{Input:} $\Dbias, \Dunbias$, Classifier and Generative Model Architectures \& Hyperparameters\\
  \textbf{Output:} Generative Model Parameters $\theta$ \\
\begin{algorithmic}[1]
\STATE $\triangleright$ Phase 1: Estimate importance weights
  \STATE \label{line:bc}Learn binary classifier $c$ for distinguishing $(\Dbias,Y=0)$ vs. $(\Dunbias,Y=1)$
  \STATE \label{line:iw_bias}Estimate importance weight $\hat{w}(\bx)\leftarrow \frac{c(Y=1 \vert \bx)}{c(Y=0 \vert \bx)}$ for all $\bx \in \Dbias$ (using Eq.~\ref{eq:imp_wts})
  \STATE\label{line:iw_unbias}Set importance weight $\hat{w}(\bx) \leftarrow 1$ for all $\bx \in \Dunbias$
 \STATE $\triangleright$ Phase 2: Minibatch gradient descent on $\theta$ based on weighted loss
 \STATE \label{line:learn_start}Initialize model parameters $\theta$ at random
 \STATE Set full dataset $\cD \leftarrow  \Dbias\cup\Dunbias$
  \WHILE{training}
  \STATE Sample a batch of points $B$ from $\cD$ at random
  \STATE \label{line:mb_loss} Set loss $\cL(\theta; \cD) \leftarrow \frac{1}{\vert B \vert}\sum_{\bx_i \in B}\hat{w}(\bx_i)\ell(\bx_i, \theta)$
  \STATE Estimate gradients $\nabla_\theta \cL(\theta; \cD)$ and update parameters $\theta$ based on optimizer update rule
  \ENDWHILE \label{line:learn_end}
   
  \STATE {\bfseries return } $\theta$
\end{algorithmic}
\end{algorithm}

Our overall procedure is summarized in Algorithm~\ref{alg:fair_gen}.
We use deep neural networks for parameterizing the binary classifier and the generative model.
Given a biased and reference dataset along with the network architectures and other standard hyperparameters (e.g., learning rate, optimizer etc.), we first learn a probabilistic binary classifier (Line 2). 
The learned classifier can provide importance weights for the datapoints from $\Dbias$ via estimates of the density ratios (Line 3). 
For the datapoints from $\Dunbias$, we do not need to perform any reweighting and set the importance weights to 1 (Line 4). 
Using the combined dataset $\Dbias \cup \Dunbias$, we then learn the generative model $p_\theta$ where the minibatch loss for every gradient update weights the contributions from each datapoint (Lines 6-12).

For a practical implementation, it is best to account for some diagnostics 
and best practices while executing Algorithm~\ref{alg:fair_gen}.
For density ratio estimation, we test that the classifier is calibrated on a held out set.
This is a necessary (but insufficient) check for the estimated density ratios to be meaningful.
If the classifier is miscalibrated, we can apply standard recalibration techniques such as Platt scaling before estimating the importance weights. Furthermore, while optimizing the model using a weighted objective, there can be an increased variance across the loss contributions from each example in a minibatch due to importance weighting.
We did not observe this in our experiments, but techniques such as normalization of weights within a batch can potentially help control the unintended variance introduced within a batch~\citep{sugiyama2012density}.

\begin{figure*}[ht]
\centering
\subfigure[single, \texttt{bias}=0.9]{\includegraphics[width=.32\linewidth]
{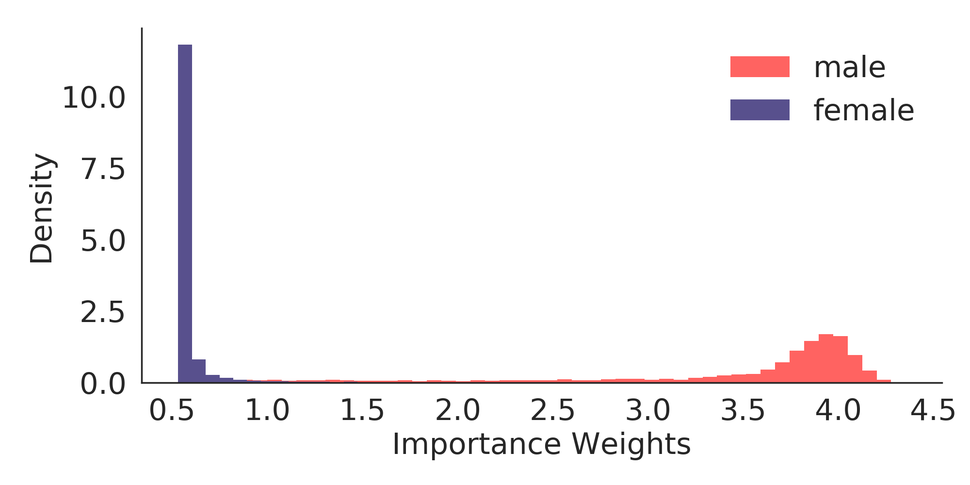}\label{fig:dre_90}}
\subfigure[single, \texttt{bias}=0.8]{\includegraphics[width=.32\linewidth]
{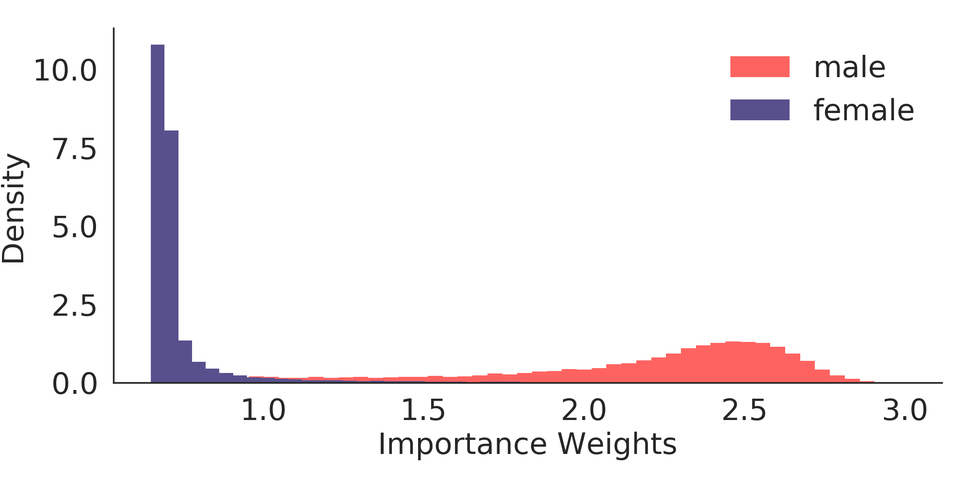}\label{fig:dre_80}}
\subfigure[multi]{\includegraphics[width=.32\linewidth]
{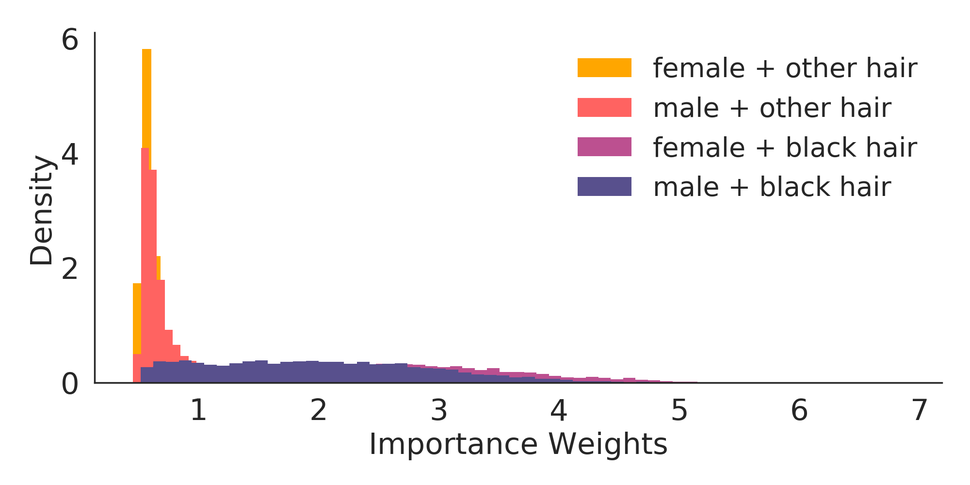}\label{fig:dre_multi}}
\caption{Distribution of importance weights for different latent subgroups. On average, The underrepresented subgroups are upweighted while the overrepresented subgroups are downweighted.}
\label{fig:imp_weights}
\end{figure*}
\textbf{Theoretical Analysis.} The performance of Algorithm~\ref{alg:fair_gen} critically depends on the quality of estimated density ratios, which in turn is dictated by the training of the binary classifier. 
We define the expected negative cross-entropy (NCE) objective for a classifier $c$ as:
\begin{align}
NCE(c) &:= \frac{1}{\gamma+1}\bbE_{\punbias(\bx)} [\log c(Y=1 \vert \bx)] \nonumber \\
&+ \frac{\gamma}{\gamma+1}\bbE_{\pbias(\bx)}[\log c(Y=0 \vert \bx)].
\end{align}

In the following result, we characterize the NCE loss for the Bayes optimal classifier.

\begin{theorem}\label{thm:optimal_nce}
Let $\mathcal{Z}$ denote a set of unobserved bias variables.
Suppose there exist two joint distributions $\pbias(\bx, \bz)$ and $\punbias(\bx, \bz)$ over $\bx \in \mathcal{X}$ and $\bz \in \mathcal{Z}$. Let $\pbias(\bx)$ and $\pbias(\bz)$ denote the marginals over $\bx$ and $\bz$ for the joint $\pbias(\bx, \bz)$ and similar notation for the joint $\punbias(\bx, \bz)$: 
\begin{align}\label{eq:cond_match}
    \pbias(\bx \vert \bz=k) &= \punbias(\bx \vert \bz=k) \; \;\forall k
\end{align}
and 
$\pbias(\bx \vert \bz=k)$, $\pbias(\bx \vert \bz=k')$ have disjoint supports for $k\neq k'$.
Then, the negative cross-entropy of the Bayes optimal classifier $c^\ast$ is given as:
\begin{align}\label{eq:optimal_nce}
NCE(c^\ast) &= \frac{1}{\gamma+1}\bbE_{\punbias(\bz)} \left[\log \frac{1}{\gamma b(\bz)+1}\right] \nonumber\\
&+ \frac{\gamma}{\gamma+1}\bbE_{\pbias(\bz)} \left[\log \frac{\gamma b(\bz)}{\gamma b(\bz)+1}\right].
\end{align}
where $b(\bz) = \pbias(\bz) / \punbias(\bz)$.
\end{theorem}
\begin{proof}
See Supplement~\ref{app:proof}.
\end{proof}
For example, as we shall see in our experiments in the following section, the inputs $\bx$ can correspond to face images, whereas the unobserved $\bz$ represents sensitive bias factors for a subgroup such as gender or ethnicity.
The proportion of examples $\bx$ belonging a subgroup can differ across the biased and reference datasets with the relative proportions given by $b(\mathbf{z})$.
Note that the above result only requires knowing these relative proportions and not the true $\bz$ for each $\bx$.
The practical implication is that under the assumptions of Theorem~\ref{thm:optimal_nce}, we can check the quality of density ratios estimated by an arbitrary learned classifier $c$ by comparing its empirical NCE with the theoretical NCE of the Bayes optimal classifier in Eq.~\ref{eq:optimal_nce} (see Section~\ref{sec:dre_analysis}).
\begin{figure*}[t]
\centering
\subfigure[
Samples generated via importance reweighting with subgroups separated by the orange line. For the 100 samples above, the classifier concludes 52 females and 48 males.
]{\includegraphics[width=.9\linewidth]
{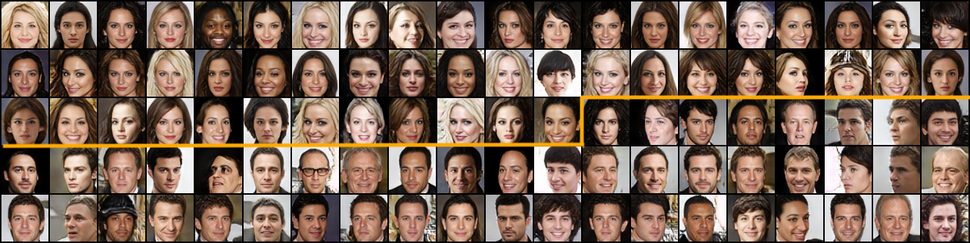}}
\subfigure[Fairness Discrepancy]{\includegraphics[width=.47\linewidth]
{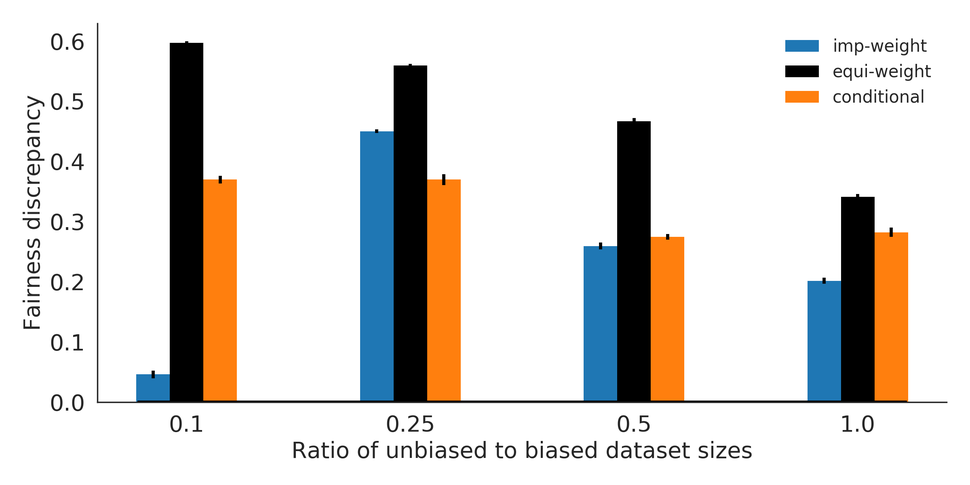}}
\subfigure[FID]{\includegraphics[width=.47\linewidth]
{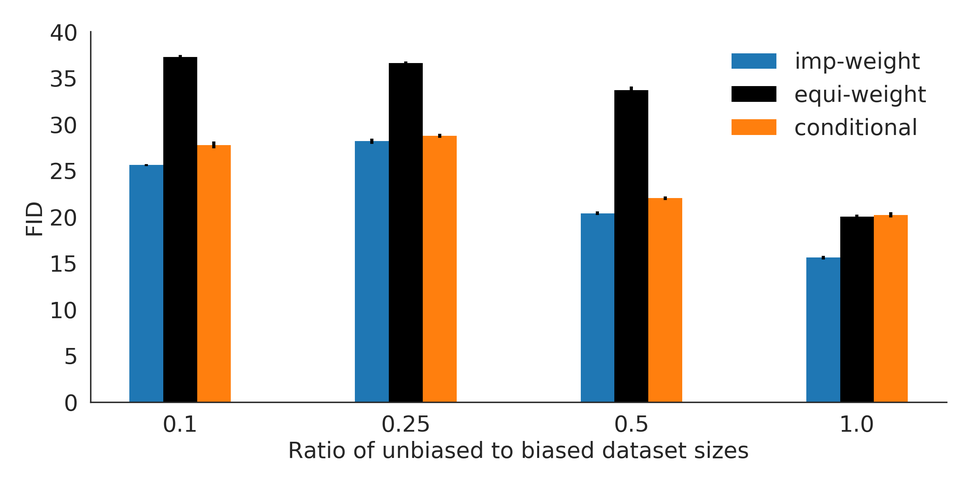}}
\caption{Single-Attribute Dataset Bias Mitigation for \texttt{bias}=0.9. Lower discrepancy and FID is better. Standard error in (b) and (c) over 10 independent evaluation sets of 10,000 samples each drawn from the models. We find that on average, \texttt{imp-weight} outperforms the \texttt{equi-weight} baseline by 49.3\% and the \texttt{conditional} baseline by 25.0\% across all reference dataset sizes for bias mitigation.
}
\label{fig:90_10_results}
\end{figure*}

\section{Empirical Evaluation}\label{sec:exp}
In this section, we are interested in investigating two broad questions empirically:
\begin{enumerate}
    \item How well can we estimate density ratios for the proposed weak supervision setting?
    \item How effective is the reweighting technique for learning fair generative models on the fairness discrepancy metric proposed in Section~\ref{sec:evaluation_metrics}?
\end{enumerate}
We further demonstrate the usefulness of our generated data in downstream applications such as data augmentation for learning a fair classifier in Supplement~\ref{sec:downstream}.

\textbf{Dataset.} We consider the CelebA \citep{liu2015faceattributes} dataset, which is commonly used for benchmarking deep generative models and comprises of images of faces with 40 labeled binary attributes. 
We use this attribute information to construct 3 different settings for partitioning the full dataset into $\Dbias$ and $\Dunbias$.

\begin{itemize}
    \item \textbf{Setting 1 (single, \texttt{bias}=0.9):} We set $\bz$ to be a single bias variable corresponding to ``gender" with values 0 (female) and 1 (male) and $b(\bz=0)=0.9$.
    
    Specifically, this means that $\Dunbias$ contains the same fraction of male and female images whereas $\Dbias$ contains 0.9 fraction of females and rest as males. 
    \item \textbf{Setting 2 (single, \texttt{bias}=0.8):} We use same bias variable (gender) as Setting 1 with $b(\bz=0)=0.8$.
    \item \textbf{Setting 3 (multi):} We set $\bz$ as two bias variables corresponding to ``gender" and ``black hair". In total, we have 4 subgroups: females without black hair (00), females with black hair (01), males without black hair (10), and males with black hair (11). We set $b(\bz=00)=0.437, b(\bz=01)=0.063,b(\bz=10)=0.415, b(\bz=11)=0.085$.
\end{itemize}
We emphasize that the attribute information is used only for designing controlled biased and reference datasets and faithful evaluation. 
Our algorithm does not explicitly require such labeled information. Additional information on constructing the dataset splits can be found in Supplement~\ref{sec:dataset}.

\textbf{Models. }
We train two classifiers for our experiments: (1) the attribute (e.g. gender) classifier which we use to assess the level of bias present in our final samples; and (2) the density ratio classifier. For both models, we use a variant of ResNet18 \citep{he2016deep} on the standard train and validation splits of CelebA. For the generative model, we used a BigGAN \citep{brock2018large} trained to minimize the hinge loss \citep{lim2017geometric,tran2017hierarchical} objective. Additional details regarding the architectural design and hyperparameters in Supplement~\ref{sec:hyperparams}.

\subsection{Density Ratio Estimation via Classifier}\label{sec:dre_analysis}

For each of the three experiments settings, we can evaluate the quality of the estimated density ratios by comparing empirical estimates of the cross-entropy loss of the density ratio classifier with the cross-entropy loss of the Bayes optimal classifier derived in Eq.~\ref{eq:optimal_nce}.
We show the results in Table~\ref{table:dre_analysis} for \texttt{perc}=1.0 where we find that the two losses are very close, suggesting that we obtain high-quality density ratio estimates that we can use for subsequently training fair generative models.
In Supplement~\ref{sec:calib_acc}, we show a more fine-grained analysis of the 0-1 accuracies and calibration of the learned models. 

\begin{table}[ht]
\centering
\begin{tabular}{l|c|c}
\toprule
\textbf{Model} & \textbf{Bayes optimal} & \textbf{Empirical} \\
\midrule
single, \texttt{bias}=0.9 & 0.591 & 0.605\\ 
single, \texttt{bias}=0.8 & 0.642 & 0.650\\ 
multi & 0.619 & 0.654\\ 
\bottomrule
\end{tabular}
\caption{Comparison between the cross-entropy loss of the Bayes classifier and learned density ratio classifier.}
\label{table:dre_analysis}
\end{table}

In Figure~\ref{fig:imp_weights}, we show the distribution of our importance weights for the various latent subgroups. We find that across all the considered settings, the underrepresented subgroups (e.g., males in Figure~\ref{fig:dre_90}, \ref{fig:dre_80}, females with black hair in \ref{fig:dre_multi}) are upweighted on average (mean density ratio estimate $>$ 1), while the overrepresented subgroups are downweighted on average (mean density ratio estimate $<$ 1). Also, as expected, the density ratio estimates are closer to 1 when the bias is low (see Figure~\ref{fig:dre_90} v.s. \ref{fig:dre_80}).

\begin{figure*}[t]
\centering
\subfigure[Samples generated via importance reweighting. For the 100 samples above, the classifier concludes 37 females and 20 males without black hair, 22 females and 21 males with black hair.
]{\includegraphics[width=.9\linewidth]
{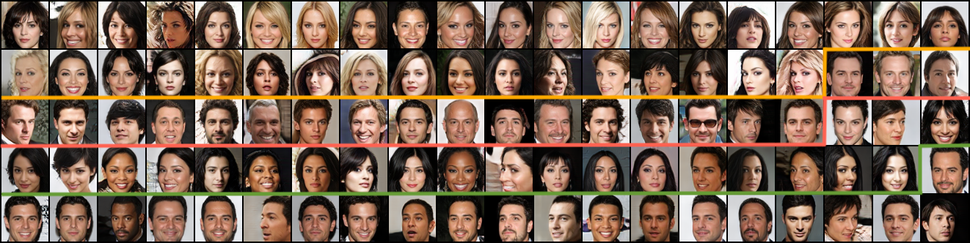}}
\subfigure[Fairness Discrepancy]{\includegraphics[width=.47\textwidth]
{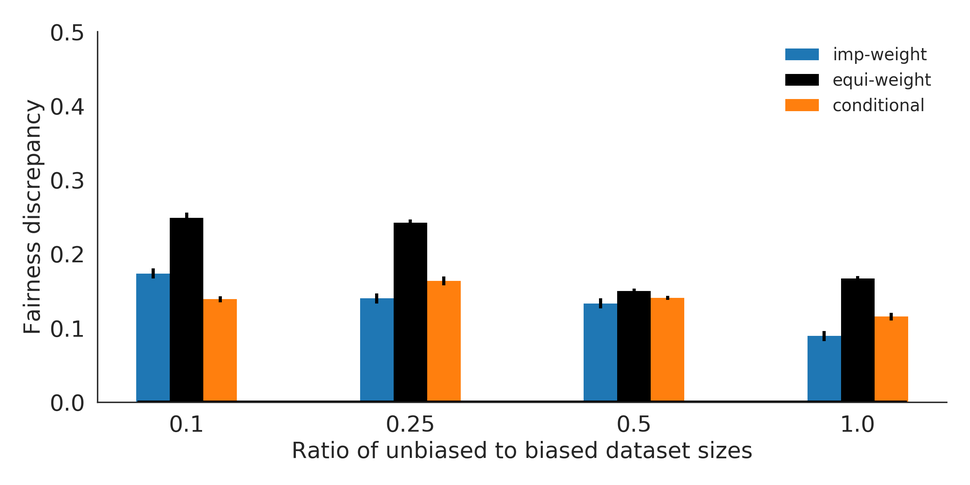}}
\subfigure[FID]{\includegraphics[width=.47\textwidth]
{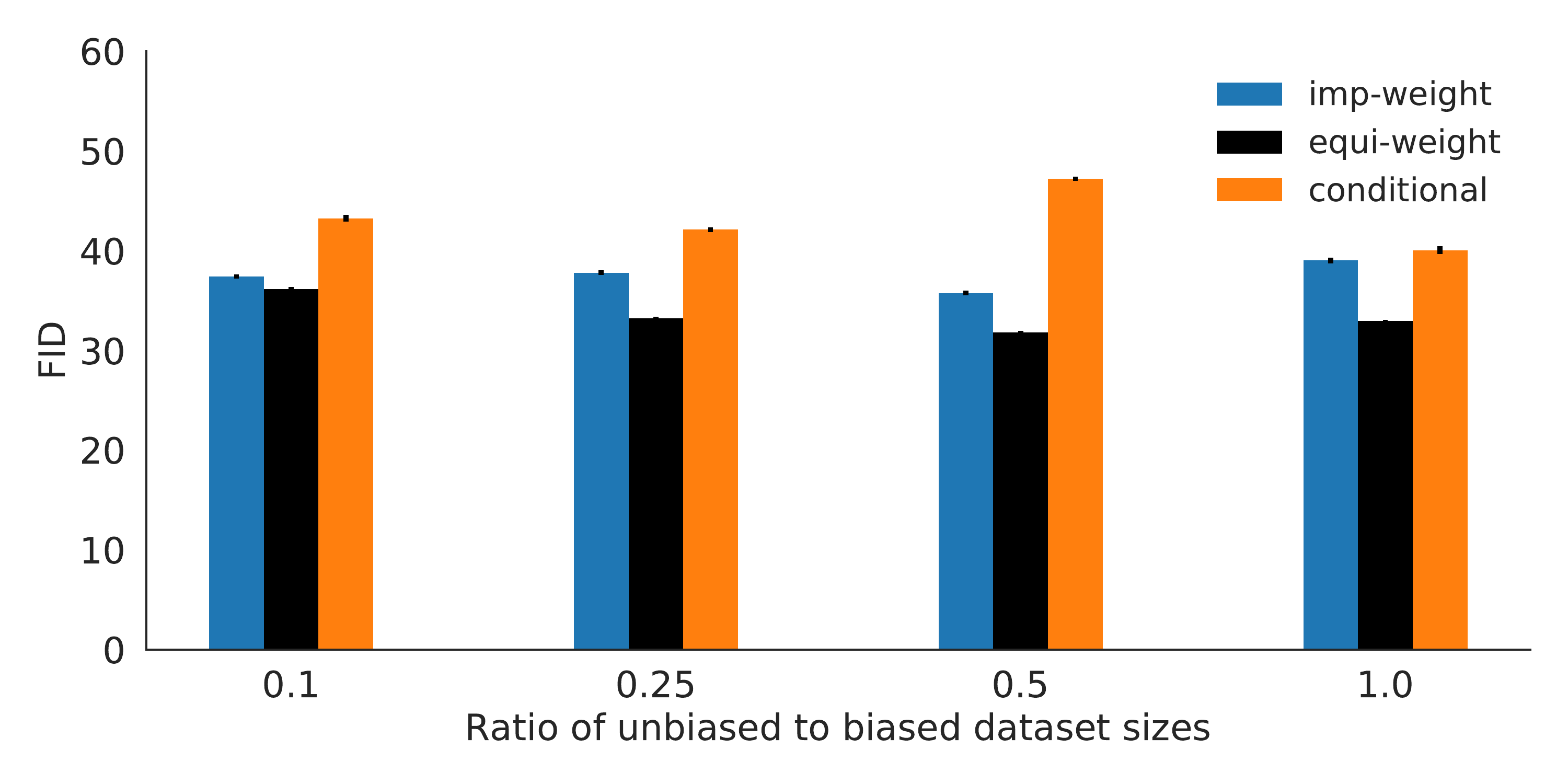}}
\caption{Mult-Attribute Dataset Bias Mitigation. Standard error in (b) and (c) over 10 independent evaluation sets of 10,000 samples each drawn from the models. Lower discrepancy and FID is better. We find that on average, \texttt{imp-weight} outperforms the \texttt{equi-weight} baseline by 32.5\% and the \texttt{conditional} baseline by 4.4\% across all reference dataset sizes for bias mitigation.
}
\label{fig:multi_results}
\end{figure*}

\subsection{Fair Data Generation}
We compare our importance weighted approach against three baselines: (1) \texttt{equi-weight}: a BigGAN trained on the full dataset $\Dunbias \cup \Dbias$ that weighs every point equally; (2) \texttt{reference-only}: a BigGAN trained on the reference dataset $\Dunbias$; and (3) \texttt{conditional}: a conditional BigGAN where the conditioning label indicates whether a data point $\bx$ is from $\Dunbias (y=1)$ or $\Dbias (y=0)$.
In all our experiments, the \texttt{reference-only} variant which  only uses the reference dataset $\Dunbias$ for learning however failed to give any recognizable samples.
For a clean presentation of the results due to other methods, we hence ignore this baseline in the results below and defer the reader to the supplementary material for further results.

We also vary the size of the balanced dataset $\Dunbias$ relative to the unbalanced dataset size $\vert\Dbias\vert$: \texttt{perc} = \{0.1, 0.25, 0.5, 1.0\}. Here, \texttt{perc} = 0.1 denotes $\vert\Dunbias\vert$ = 10\% of $\vert\Dbias\vert$ and \texttt{perc} = 1.0 denotes $\vert\Dunbias\vert = \vert\Dbias\vert$.

\subsubsection{Single Attribute Splits}\label{sec:gender_exp}
We train our attribute (gender) classifier for evaluation on the entire CelebA training set, and achieve a level of 98\% accuracy on the held-out set. For each experimental setting, we evaluate bias mitigation based on the fairness discrepancy metric (Eq.~\ref{eq:fairness_discrepancy}) and also report sample quality based on FID~\citep{heusel2017gans}.

For the \texttt{bias} = 0.9
split, we show the samples generated via \texttt{imp-weight} in Figure~\ref{fig:90_10_results}a and the resulting fairness discrepancies in Figure~\ref{fig:90_10_results}b. Our framework generates samples that are slightly lower quality than \texttt{equi-weight} baseline samples shown in Figure~\ref{fig:baseline}, but is able to produce almost identical proportion of samples across the two genders. Similar observations hold for \texttt{bias} = 0.8, as shown in Figure~\ref{fig:80_20_results} in the supplement. We refer the reader to Supplement~\ref{sec:supp_80_20} for corresponding results and analysis, as well as for additional results on the Shapes3D dataset \citep{3dshapes18}.

\subsubsection{Multi-Attribute Split}
We conduct a similar experiment with a multi-attribute split based on gender and the presence of black hair.
The attribute classifier for the purpose of evaluation is now trained with a 4-way classification task instead of 2, and achieves an accuracy of roughly 88\% on the test set.  

Our model produces samples as shown in Figure~\ref{fig:multi_results}a with the discrepancy metrics shown in Figures~\ref{fig:multi_results}b, c respectively.
Even in this challenging setup involving two latent bias factors, we find that the importance weighted approach again outperforms the baselines in almost all cases in mitigating bias in the generated data while admitting only a slight deterioration in image quality overall. 
\section{Related Work}
\textbf{Fairness \& generative modeling. } 
There is a rich body of work in fair ML, which focus on different notions of fairness (e.g. demographic parity, equality of odds and opportunity) and study methods by which models can perform tasks such as classification in a non-discriminatory way~\citep{barocas-hardt-narayanan,dwork2012fairness,heidari2018fairness,du2018data}. Our focus is in the context of fair generative modeling.
The vast majority of related work in this area is centered around fair and/or privacy preserving \textit{representation learning}, which exploit tools from adversarial learning and information theory among others~\citep{zemel2013learning, edwards2015censoring, louizos2015variational, beutel2017data,song2018learning, adel2019one}. 
A unifying principle among these methods is such that a discriminator is trained to perform poorly in predicting an outcome based on a protected attribute. 
~\citet{ryu2017inclusivefacenet} considers transfer learning of race and gender identities as a form of weak supervision for predicting other attributes on datasets of faces.
While the end goal for the above works is classification, our focus is on data generation in the presence of dataset bias and we do not require explicit supervision for the protected attributes.

The most relevant prior works in data generation are  FairGAN~\citep{xu2018fairgan} and FairnessGAN~\citep{sattigeri2019fairness}.
The goal of both methods is to generate \textit{fair} datapoints and their labels as a preprocessing technique. This allows for learning a useful downstream classifier and obscures information about protected attributes. 
Again, these works are not directly comparable to ours as we do not assume explicit supervision regarding the protected attributes during training, and our goal is fair generation given unlabelled biased datasets where the bias factors are latent. Another relevant work is DB-VAE~\citep{amini2019uncovering}, which utilizes a VAE to learn the latent structure of sensitive attributes, and in turn employs importance weighting based on this structure to mitigate bias in downstream classifiers. Contrary to our work, these importance weights are used to directly sample (rare) data points with higher frequencies with the goal of training a classifier (e.g. as in a facial detection system), as opposed to fair generation.

\textbf{Importance reweighting. } Reweighting datapoints is a common algorithmic technique for problems such as dataset bias and class imbalance~\citep{byrd2018weighted}.
It has often been used in the context of fair classification~\citep{calders2009building}, for example, \cite{kamiran2012data} details \textit{reweighting} as a way to remove discrimination without relabeling instances.
For reinforcement learning, \citet{doroudi2017importance} used an importance sampling approach for selecting fair policies. There is also a body of work on fair clustering \citep{chierichetti2017fair,backurs2019scalable,bera2019fair,schmidt2018fair} which ensure that the clustering assignments are balanced with respect to some sensitive attribute.

\textbf{Density ratio estimation using classifiers. } The use of classifiers for estimating density ratios has a rich history of prior works across ML~\citep{sugiyama2012density}. 
For deep generative modeling, density ratios estimated by classifiers have been used for expanding the class of various learning objectives~\citep{nowozin2016f,mohamed2016learning,grover2018boosted}, evaluation metrics based on two-sample tests~\citep{gretton2007kernel,bowman2015generating,lopez2016revisiting,danihelka2017comparison,rosca2017variational,im2018quantitatively,gulrajani2018towards}, or improved Monte Carlo inference via these models~\citep{grover2019debiasing,azadi2018discriminator,turner2018metropolis,tao2018chi}. 
\citet{grover2019debiasing} use importance reweighting for mitigating model bias between $\pdata$ and $p_\theta$.

Closest related is the proposal of \citet{diesendruck2018importance} to use importance reweighting for learning generative models where training and test distributions differ, but explicit importance weights are provided for at least a subset of the training examples.
We consider a more realistic, weakly-supervised setting where we estimate the importance weights using a small reference dataset.
Finally, another related line of work in domain translation via generation considers learning via multiple datasets~\citep{cycle-gan, choi2018generative,grover2020alignflow} and it would be interesting to consider issues due to dataset bias in those settings in future work.

\section{Discussion}
\label{sec:discussion}
Our work presents an initial foray into the field of fair image generation with weak supervision, and we stress the need for caution in using our techniques and interpreting the empirical findings.
For scaling our evaluation, we proposed metrics that relied on a pretrained attribute classifier for inferring the bias in the generated data samples.
The classifiers we considered are highly accurate on all subgroups, but can have blind spots especially when evaluated on generated data.
For future work, we would like to investigate conducting human evaluations to mitigate such issues during evaluation~\citep{grgic2018human}.

As another case in point, our work calls for rethinking sample quality metrics for generative models in the presence of dataset bias~\citep{mitchell2019model}.
On one hand, our approach \textit{increases} the diversity of generated samples in the sense that the different subgroups are more balanced;
at the same time, however, variation across other image features 
\textit{decreases} because the newly generated underrepresented samples are learned from a smaller dataset of underrepresented subgroups.
Moreover, standard metrics such as FID even when evaluated with respect to a reference dataset, could exhibit a relative preference for models trained on larger datasets with little or no bias correction to avoid even slight compromises on perceptual sample quality.

More broadly, this work is yet another reminder that we must be mindful of the decisions made at each stage in the development and deployment of ML systems~\citep{abebe2020roles}. Factors such as the dataset used for training~\citep{gebru2018datasheets,sheng2019woman,jo2020lessons} or algorithmic decisions such as the loss function or evaluation metric~\citep{hardt2016equality,buolamwini2018gender,kim2018interpretability,liu2018delayed,hashimoto2018fairness}, among others, may have undesirable consequences. Becoming more aware of these downstream impacts will help to mitigate the potentially discriminatory nature of our present-day systems~\citep{kaeser2020positionality}.

\section{Conclusion}
We considered the task of fair data generation given access to a (potentially small) reference dataset and a large biased dataset.
For data-efficient learning, we proposed an importance weighted objective that corrects bias by reweighting the biased datapoints.
These weights are estimated by a binary classifier.
Empirically, we showed that our technique outperforms baselines by up to 34.6\% on average in reducing dataset bias on CelebA without incurring a significant reduction in sample quality. We provide reference implementations in PyTorch \citep{paszke2017automatic}, and the codebase for this work is open-sourced at \texttt{https://github.com/ermongroup/fairgen}.

In the future, it would be interesting to explore whether even weaker forms of supervision would be possible for this task, e.g., when the biased dataset has a somewhat disjoint but \textit{related} support from the small, reference dataset -- this would be highly reflective of the diverse data sources used for training many current and upcoming large-scale ML systems~\citep{ratner2017snorkel}.
\section*{Acknowledgements}
We are thankful to Hima Lakkaraju, Daniel Levy, Mike Wu, Chris Cundy, and Jiaming Song for insightful discussions and feedback. KC is supported by the NSF GRFP, Qualcomm Innovation Fellowship, and Stanford Graduate Fellowship, and AG is supported by the MSR Ph.D. fellowship, Stanford Data Science scholarship, and Lieberman fellowship. This research was funded by NSF (\#1651565, \#1522054, \#1733686), ONR (N00014-19-1-2145), AFOSR (FA9550-19-1-0024), and Amazon AWS. 
\bibliography{references}

\bibliographystyle{icml2020}
\raggedbottom

\pagebreak
\newpage
\renewcommand\thesection{\Alph{section}}
\setcounter{section}{0}
\onecolumn
\section*{Supplementary Material}
\section{Proof of Theorem~\ref{thm:optimal_nce}}\label{app:proof}

\begin{proof}

Since $\pbias(\bx \vert \bz=k)$ and $\pbias(\bx \vert \bz=k')$ have disjoint supports for $k\neq k'$, we know that for all $\bx$, there exists a deterministic mapping $f: \mathcal{X} \to \mathcal{Z}$ such that $\pbias(\bx \vert \bz=f(\bx))>0$.

 Further, for all $\tilde{\bx} \not \in f^{-1}(\bz)$:
\begin{align}
    \label{eq:many_to_one_1}\pbias(\tilde{\bx} \vert \bz=f(\bx)) &= 0;\\
    \label{eq:many_to_one_2}\punbias(\tilde{\bx} \vert \bz=f(\bx)) &= 0.
\end{align}

Combining Eqs.~\ref{eq:many_to_one_1},\ref{eq:many_to_one_2} above with the assumption in Eq.~\ref{eq:cond_match}, we can simplify the density ratios as:
\begin{align}\label{eq:bias_ratio}
    \frac{\pbias(\bx)}{\punbias(\bx)} &= \frac{\int_{\bz}\pbias(\bx \vert \bz ) \pbias(\bz)\mathrm{d}\bz}{\int_{\bz}\punbias(\bx \vert \bz) \punbias(\bz)\mathrm{d}\bz} \\
    &= \frac{\pbias(\bx \vert f(\bx)) \pbias(f(\bx))}{\punbias(\bx \vert f(\bx)) \punbias(f(\bx))}  \;\; \text{(using Eqs.~\ref{eq:many_to_one_1},\ref{eq:many_to_one_2})}\\
    &= \frac{\pbias(f(\bx))}{\punbias(f(\bx))}\;\; \text{(using Eq.~\ref{eq:cond_match})}\\
    &= b(f(\bx)).
\end{align}
From Eq.~\ref{eq:imp_wts} and Eq.~\ref{eq:bias_ratio}, the Bayes optimal classifier $c^\ast$ can hence be expressed as:
\begin{align}\label{eq:bayes_opt_dre}
    c^\ast(Y=1 \vert \bx) &= \frac{1}{\gamma b(f(\bx))+1}.
\end{align}

The optimal cross-entropy loss of a binary classifier $c$ for density ratio estimation (DRE) can then be expressed as:
\begin{align}\label{eq:ece_loss}
NCE(c^\ast) &= \frac{1}{\gamma+1}\bbE_{\punbias(\bx)} [\log c^\ast(Y=1 \vert \bx)] + \frac{\gamma}{\gamma+1}\bbE_{\pbias(\bx)}[\log c^\ast(Y=0 \vert \bx)]\\
&= \bbE_{\punbias(\bx)} \left[\log \frac{1}{\gamma b(f(\bx))+1}\right] + \frac{\gamma}{\gamma+1}\bbE_{\pbias(\bx)}\left[\log \frac{\gamma b(f(\bx))}{\gamma b(f(\bx))+1}\right]  \;\; \text{(using Eq.~\ref{eq:bayes_opt_dre})}\\
&= \frac{1}{\gamma+1}\bbE_{\punbias(\bz)} \bbE_{\punbias(\bx \vert \bz)} \left[\log \frac{1}{\gamma b(f(\bx))+1}\right] + \frac{\gamma}{\gamma+1}\bbE_{\pbias(\bz)}  \bbE_{\pbias(\bx \vert \bz)}\left[\log \frac{\gamma b(f(\bx))}{\gamma b(f(\bx))+1}\right]\\
&= \frac{1}{\gamma+1}\bbE_{\punbias(\bz)} \left[\log \frac{1}{\gamma b(\bz)+1}\right] + \frac{\gamma}{\gamma+1}\bbE_{\pbias(\bz)} \left[\log \frac{\gamma b(\bz)}{\gamma b(\bz)+1}\right]  \;\; \text{(using Eqs.~\ref{eq:many_to_one_1},\ref{eq:many_to_one_2})}. 
\end{align}
\end{proof}

\section{Dataset Details}
\subsection{Dataset Construction Procedure}\label{sec:dataset}
We construct such dataset splits from the full CelebA training set using the following procedure. We initially fix our dataset size to be roughly 135K out of the total 162K based on the total number of females present in the data. Then for each level of \texttt{bias}, we partition 1/4 of males and 1/4 of females into $\Dunbias$ to achieve the 50-50 ratio. The remaining number of examples are used for $\Dbias$, where the number of males and females are adjusted to match the desired level of \texttt{bias} (e.g. 0.9). Finally at each level of reference dataset size \texttt{perc}, we discard the appropriate fraction of datapoints from both the male and female category in $\Dunbias$. For example, for \texttt{perc} = 0.5, we discard half the number of females and half the number of males from $\Dunbias$.

\subsection{FID Calculation}\label{sec:unbiased_fid}
As noted Sections 2.3 and 6, the FID metric may exhibit a relative preference for models trained on larger datasets in order to maximize perceptual sample quality, at the expense of propagating or amplifying existing dataset bias. In order to obtain an estimate of sample quality that would also incorporate a notion of fairness across sensitive attribute classes, we pre-computed the relevant FID statistics on a "balanced" construction of the CelebA dataset that matches our reference dataset $\punbias$. That is, we used all train/validation/test splits of the data such that: (1) for single-attribute, there were 50-50 portions of males and females; and (2) for multi-attribute, there were even proportions of examples across all 4 classes (females with black hair, females without black hair, males with black hair, males without black hair). We report "balanced" FID numbers on these pre-computed statistics throughout the paper.

\section{Architecture and Hyperparameter Configurations}\label{sec:hyperparams}

We used PyTorch~\cite{paszke2017automatic} for all our experiments.
Our overall experimental framework involved three different kinds of models which we describe below.

\subsection{Attribute Classifier}
We use the same architecture and hyperparameters for both the single- and multi-attribute classifiers. Both are variants of ResNet-18 where the output number of classes correspond to the dataset split (e.g. 2 classes for single-attribute, 4 classes for the multi-attribute experiment). 

\paragraph{Architecture.} 
We provide the architectural details in Table \ref{table:resnet_arch} below:
 
\begin{table}[h!]
\centering
\begin{tabular}{c|c}
\hline
\textbf{Name}& \textbf{Component}\\
\hline
conv1 & $7\times7$ conv, 64 filters. stride 2 \\
\hline
Residual Block 1 & $3 \times 3$ max pool, stride 2 \\
\hline
Residual Block 2 & 
$
\begin{bmatrix}
    3 \times 3 \text{ conv, } 128 \text{ filters} \\
    3 \times 3 \text{ conv, } 128 \text{ filters}
\end{bmatrix}
\times 2$ \\
\hline
Residual Block 3 & $
\begin{bmatrix}
    3 \times 3 \text{ conv, } 256 \text{ filters} \\
    3 \times 3 \text{ conv, } 256 \text{ filters}
\end{bmatrix}
\times 2$ \\
\hline
Residual Block 4 & $
\begin{bmatrix}
    3 \times 3 \text{ conv, } 512 \text{ filters} \\
    3 \times 3 \text{ conv, } 512 \text{ filters}
\end{bmatrix}
\times 2$ \\
\hline
Output Layer & $7 \times 7$ average pool stride 1, fully-connected, softmax \\
\hline
\end{tabular}
\caption{ResNet-18 architecture adapted for attribute classifier.}
\label{table:resnet_arch}
\end{table}

\paragraph{Hyperparameters.} During training, we use a batch size of 64 and the Adam optimizer with learning rate = 0.001. The classifiers learn relatively quickly for both scenarios and we only needed to train for 10 epochs. We used early stopping with the validation set in CelebA to determine the best model to use for downstream evaluation.

\subsection{Density Ratio Classifier}
\paragraph{Architecture.} 
We provide the architectural details in Table \ref{table:resnet_arch}.
 
\begin{table}[h!]
\centering
\begin{tabular}{c|c}
\hline
\textbf{Name}& \textbf{Component}\\
\hline
conv1 & $7\times7$ conv, 64 filters. stride 2 \\
\hline
Residual Block 1 & $3 \times 3$ max pool, stride 2 \\
\hline
Residual Block 2 & $
\begin{bmatrix}
    3 \times 3 \text{ conv, } 128 \text{ filters} \\
    3 \times 3 \text{ conv, } 128 \text{ filters}
\end{bmatrix}
\times 2$ \\
\hline
Residual Block 3 & $
\begin{bmatrix}
    3 \times 3 \text{ conv, } 256 \text{ filters} \\
    3 \times 3 \text{ conv, } 256 \text{ filters}
\end{bmatrix}
\times 2$ \\
\hline
Residual Block 4 & $
\begin{bmatrix}
    3 \times 3 \text{ conv, } 512 \text{ filters} \\
    3 \times 3 \text{ conv, } 512 \text{ filters}
\end{bmatrix}
\times 2$ \\
\hline
Output Layer & $7 \times 7$ average pool stride 1, fully-connected, softmax \\
\hline
\end{tabular}
\caption{ResNet-18 architecture adapted for attribute classifier.}
\label{table:resnet_arch_2}
\end{table}

\paragraph{Hyperparameters.} We also use a batch size of 64, the Adam optimizer with learning rate = 0.0001, and a total of 15 epochs to train the density ratio estimate classifier. 

\paragraph{Experimental Details.} We note a few steps we had to take during the training and validation procedure. Because of the imbalance in both (a) unbalanced/balanced dataset sizes and (b) gender ratios, we found that a naive training procedure encouraged the classifier to predict all data points as belonging to the biased, unbalanced dataset. To prevent this phenomenon from occuring, two minor modifications were necessary:
\begin{enumerate}
    \item We \textit{balance} the distribution between the two datasets in each minibatch: that is, we ensure that the classifier sees equal numbers of data points from the balanced ($y=1$) and unbalanced ($y=0$) datasets for each batch. This provides enough signal for the classifier to learn meaningful density ratios, as opposed to a trivial mapping of all points to the larger dataset.
    \item We apply a similar balancing technique when testing against the validation set. However, instead of balancing the minibatch, we weight the contribution of the losses from the balanced and unbalanced datasets. Specifically, the loss is computed as: 
    $$\mathcal{L} = \frac{1}{2} 
    \left(
    \frac{\text{acc}_\text{pos}}{\text{n}_\text{pos}} + 
    \frac{\text{acc}_\text{neg}}{\text{n}_\text{neg}}
    \right)
    $$
    where the subscript \texttt{pos} denotes examples from the balanced dataset ($y=1$) and \texttt{neg} denote examples from the unbalanced dataset ($y=0$).
\end{enumerate}

\subsection{BigGAN}
\paragraph{Architecture.}
The architectural details for the BigGAN are provided in Table \ref{table:biggan_arch}.
\begin{table}[!h]
\centering
\begin{tabular}{c|c}
\hline
\textbf{Generator} & \textbf{Discriminator} \\
\hline
$1 \times 1 \times 2ch$ Noise & $64 \times 64 \times 3$ Image \\
\hline
Linear $1 \times 1 \times 16ch \rightarrow 1 \times 1 \times 16ch$  & ResBlock down $1ch \rightarrow 2ch$\\
ResBlock up $16ch \rightarrow 16ch$ &  Non-Local Block ($64 \times 64$)\\
\hline
ResBlock up $16ch \rightarrow 8ch$ & ResBlock down $2ch \rightarrow 4ch$\\
\hline
ResBlock up $8ch \rightarrow 4ch$ & ResBlock down $4ch \rightarrow 8ch$\\
\hline
ResBlock up $4ch \rightarrow 2ch$ & ResBlock down $8ch \rightarrow 16ch$\\
\hline
Non-Local Block ($64 \times 64$) & ResBlock down $16ch \rightarrow 16ch$\\
ResBlock up $4ch \rightarrow 2ch$ & ResBlock $16ch \rightarrow 16ch$\\
\hline
BatchNorm, ReLU, $3 \times 3$ Conv $1ch \rightarrow 3$ & ReLU, Global sum pooling\\
\hline
Tanh & Linear $\rightarrow 1$\\
\hline
\end{tabular}
\caption{Architecture for the generator and discriminator. Notation: $ch$ refers to the channel width multiplier, which is $64$ for $64 \times 64$ CelebA images. ResBlock up refers to a Generator Residual Block in which the input is passed through a ReLU activation followed by two $3 \times 3$ convolutional layers with a ReLU activation in between. ResBlock down refers to a Discriminator Residual Block in which the input is passed through two $3 \times 3$ convolution layers with a ReLU activation in between, and then downsampled. Upsampling is performed via nearest neighbor interpolation, whereas downsampling is performed via mean pooling. ``ResBlock up/down $n \rightarrow m$'' indicates a ResBlock with $n$ input channels and $m$ output channels.}
\label{table:biggan_arch}
\end{table}
\paragraph{Hyperparameters.} We sweep over a batch size of $\{16, 32, 64, 128\}$, and the Adam optimizer with learning rate $=0.0002$, and $\beta_1=0, \beta_2=0.99$. We train the model by taking 4 discriminator gradient steps per generator step. Because the BigGAN was originally designed for scaling up class-conditional image generation, we fix all conditioning labels for the unconditional baselines (\texttt{imp-weight, equi-weight}) to the zero vector.

Additionally, we investigate the role of \textit{flattening} in the density ratios used to train the generative model. As in~\citep{grover2019debiasing}, flattening the density ratios via a power scaling parameter $\alpha \geq 0$ is defined as:
\begin{align*}
    \mathbb{E}_{\bx \sim \punbias [\ell(\bx, \theta)]} \approx \frac{1}{T} \sum_{i=1}^T w(\bx_i)^\alpha \ell(\bx_i, \theta)
\end{align*}
where $\bx_i \sim \pbias$. We perform a hyperparameter sweep over $\alpha=\{0.5, 1.0, 1.5\}$, while noting that $\alpha=0$ is equivalent to the \texttt{equi-weight} baseline (no reweighting).

\section{Density Ratio Classifier Analysis}\label{sec:calib_acc}

\begin{figure}[h!]
\centering
\subfigure[\texttt{bias=0.9}, \texttt{perc=1.0}]{\includegraphics[width=.3\linewidth]{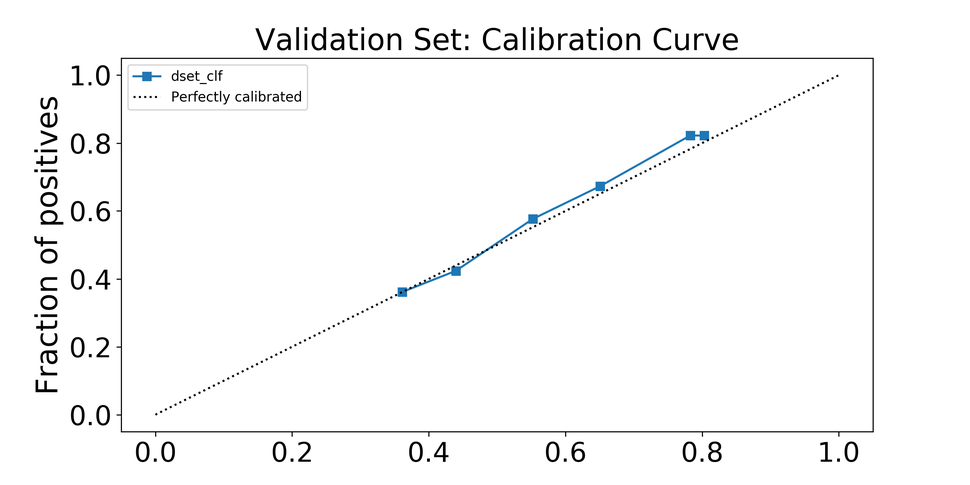}}
\subfigure[\texttt{bias=0.8}, \texttt{perc=1.0}]{\includegraphics[width=.3\linewidth]{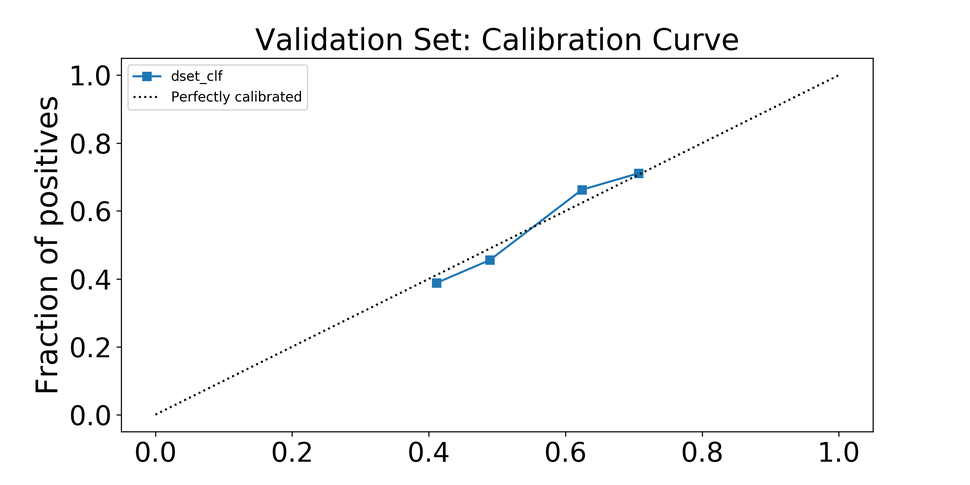}}
\subfigure[\texttt{multi}, \texttt{perc=1.0}]{\includegraphics[width=.3\linewidth]{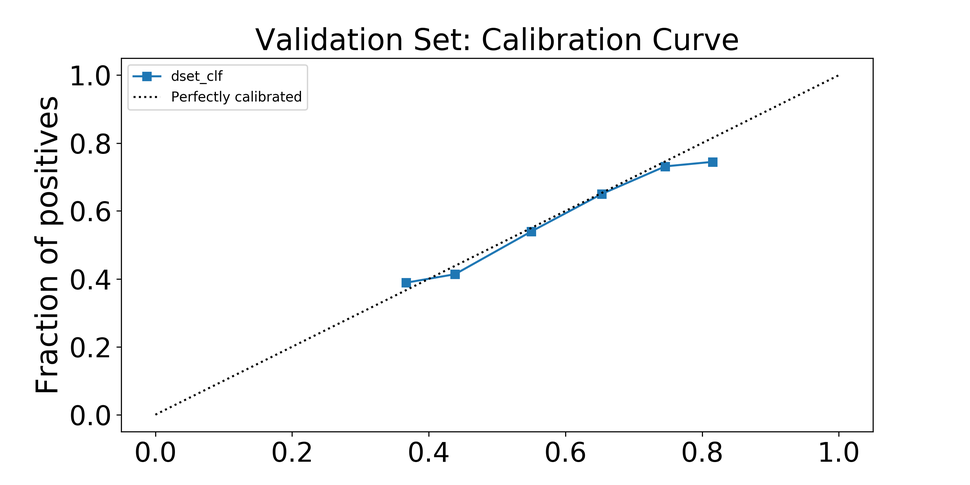}}
\subfigure[\texttt{bias=0.9}, \texttt{perc=0.5}]{\includegraphics[width=.3\linewidth]{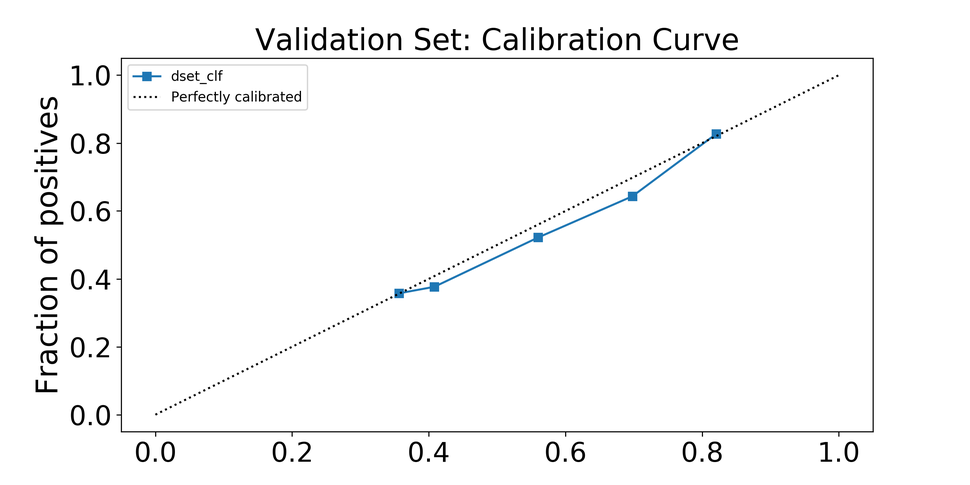}}
\subfigure[\texttt{bias=0.8}, \texttt{perc=0.5}]{\includegraphics[width=.3\linewidth]{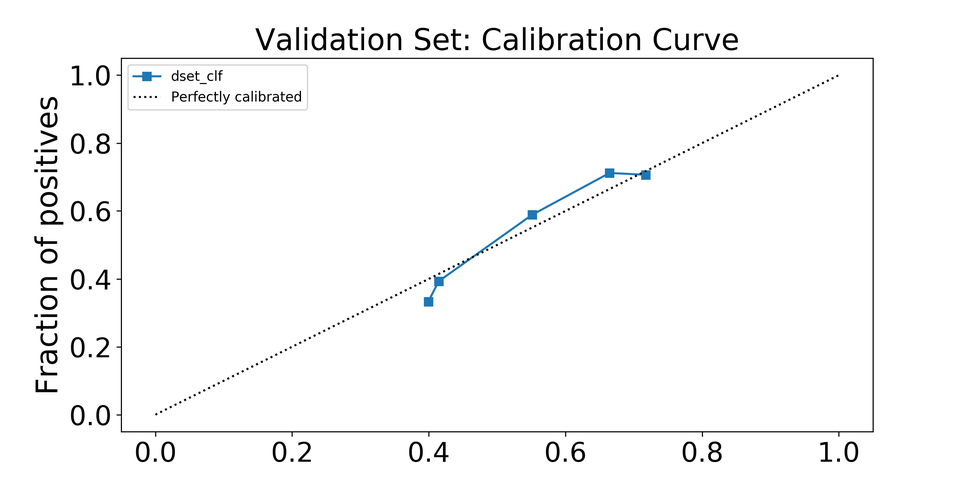}}
\subfigure[\texttt{multi}, \texttt{perc=0.5}]{\includegraphics[width=.3\linewidth]{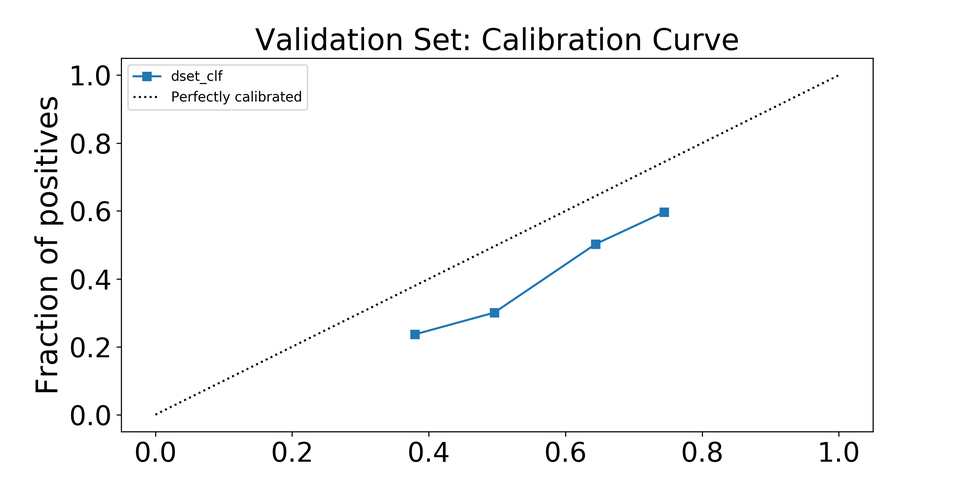}}
\subfigure[\texttt{bias=0.9},  \texttt{perc=0.25}]{\includegraphics[width=.3\linewidth]{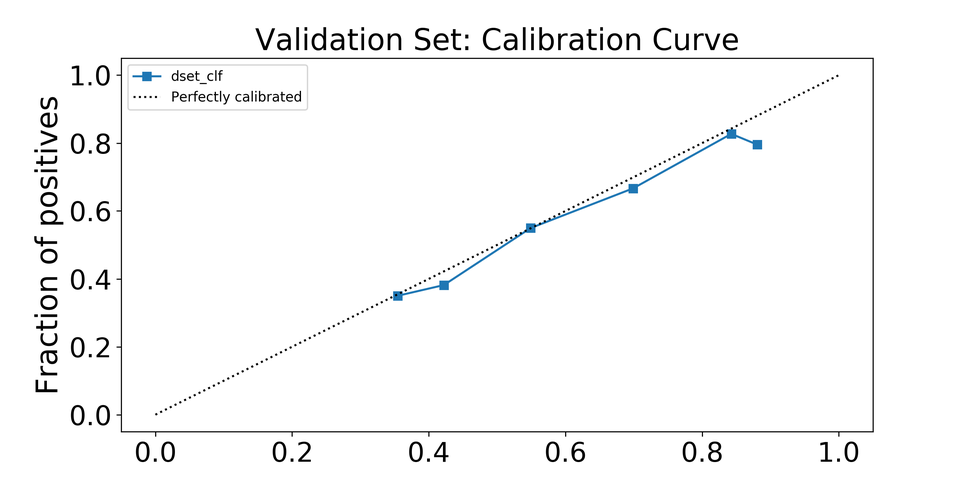}}
\subfigure[\texttt{bias=0.8}, \texttt{perc=0.25}]{\includegraphics[width=.3\linewidth]{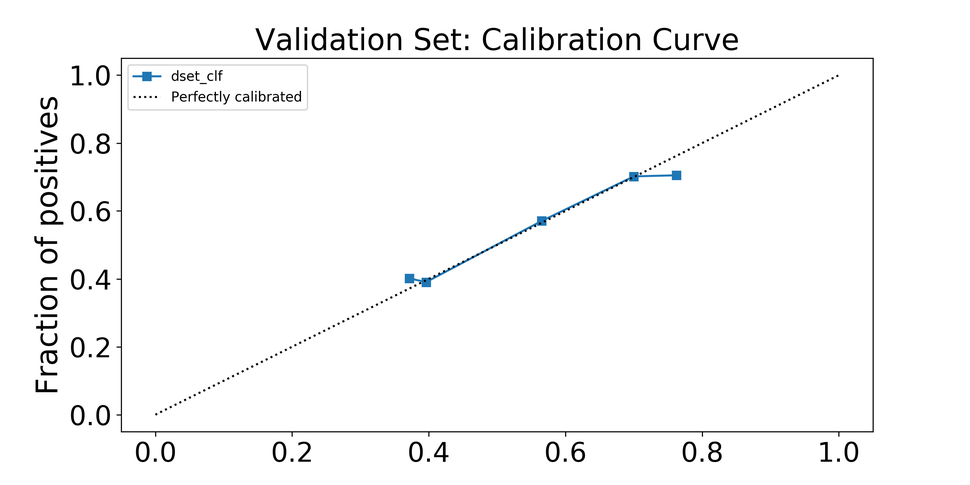}}
\subfigure[Multi \texttt{perc}=0.25]{\includegraphics[width=.3\linewidth]{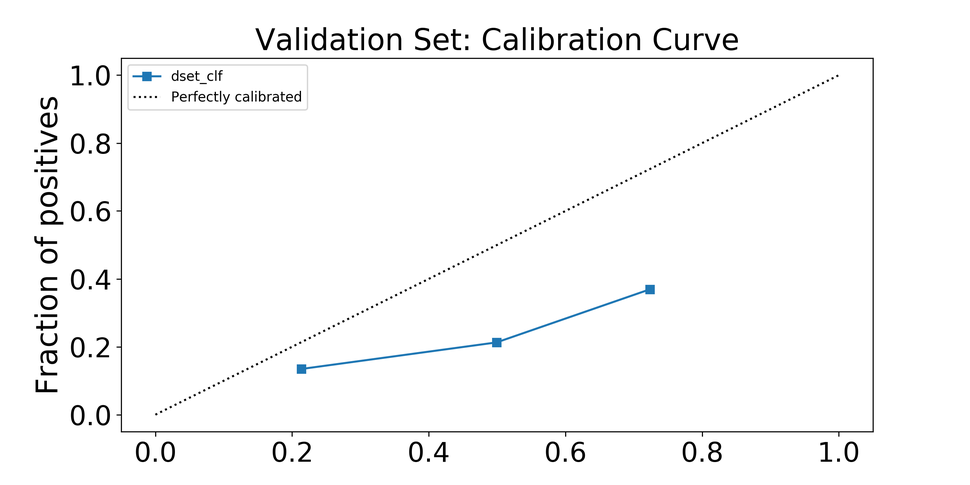}}
\subfigure[\texttt{bias=0.9},  \texttt{perc=0.1}]{\includegraphics[width=.3\linewidth]{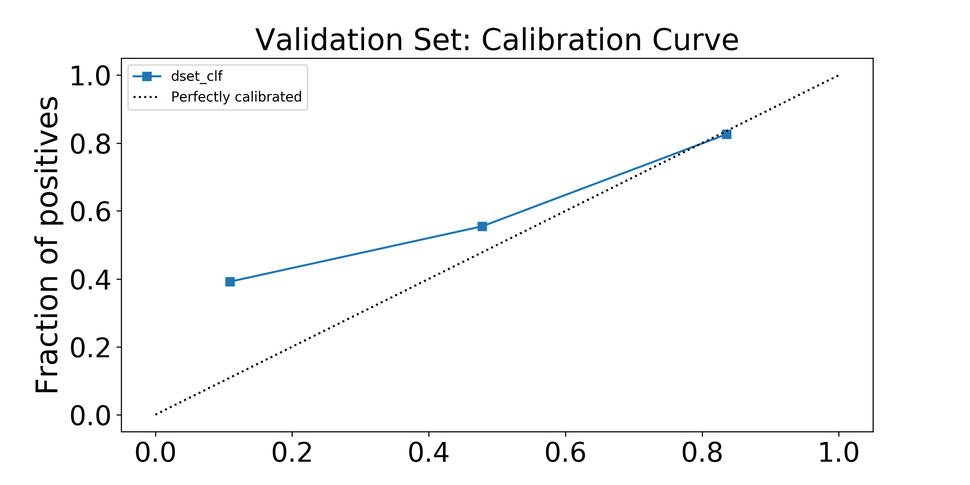}}
\subfigure[\texttt{bias=0.8}, \texttt{perc=0.1}]{\includegraphics[width=.3\linewidth]{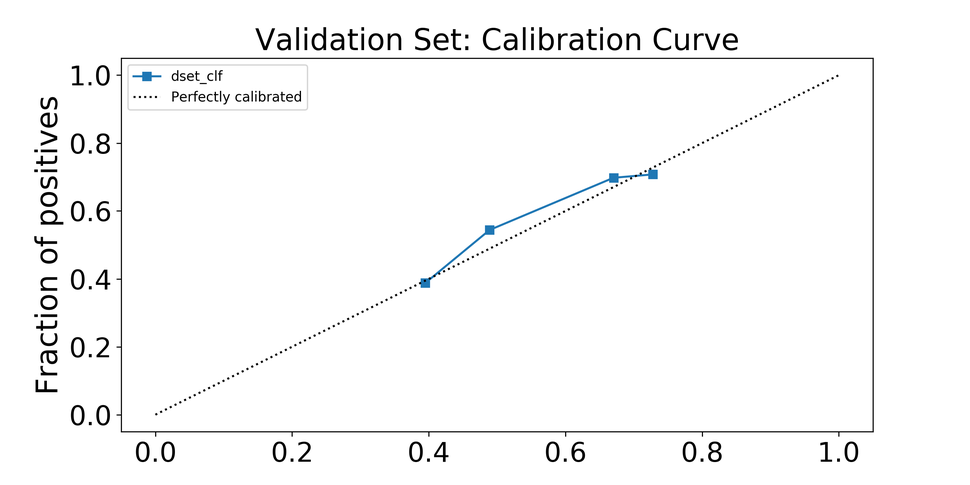}}
\subfigure[Multi \texttt{perc=0.1}]{\includegraphics[width=.3\linewidth]{new_figures/appendix/calibration/90_10_small01_calibration_curve.png}}
\caption{Calibration curves}
\label{fig:calibration}
\end{figure}

In Figure~\ref{fig:calibration}, we show the calibration curves for the density ratio classifiers for each of the $\Dunbias$ dataset sizes across all levels of bias.
As evident from the plots, most classifiers are already calibrated and did not require any post-training recalibration.

\section{Fairness Discrepancy Metric}\label{sec:fairdisc_metric}
In this section, we motivate the fairness discrepancy metric and elaborate upon its construction. Recall from Equation~\ref{eq:fairness_discrepancy} that the metric is as follows for the sensitive attributes $\bu$:
    \begin{align*}
        f(\punbias, p_\theta) = \vert \bbE_{\punbias}[p(\bu \vert \bx)] - \bbE_{p_\theta}[p(\bu \vert \bx)] \vert_2.
    \end{align*}

To gain further insight into what the metric is capturing, we rewrite the joint distribution of the sensitive attributes $\bu$ and our data $\bx$: (1) $\punbias(\bu, \bx) = p(\bu|\bx) \punbias(\bx)$ and (2)  $p_\theta(\bu, \bx) = p(\bu|\bx) p_\theta(\bx)$. 
Then, marginalizing out $\bx$ and only looking at the distribution of $\bu$, we get that  $p(\bu) = \int p(\bu,\bx) dx = \int p(\bu|\bx) p(\bx) dx = \mathbb{E}_p(\bx) p(\bu|\bx)$. Thus the fairness discrepancy metric is $|\punbias(\bu) - p_\theta(\bu)|_2$.

This derivation is informative because it allows us to relate the fairness discrepancy metric to the behavior of the (oracle) attribute classifier. Suppose we use a deterministic classifier $p(\bu|\bx)$ as in the paper: that is, we threshold at 0.5 to label all examples with $p(\bu|\bx) > 0.5$ as $\bu = 1$ (e.g. male), and $p(\bu|\bx) \leq 0.5$ as $\bu = 0$ (e.g. female). In this setting, the fairness discrepancy metric simply becomes the $\ell_2$ distance in proportions of different populations between the true (reference) dataset and the generated examples. 

It is easy to see that if we use a probabilistic classifier (without thresholding), we can obtain similar distributional discrepancies between the true (reference) data distribution and the distribution learned by $p_\theta$ such as the empirical KL.

\section{Additional Results}

\subsection{Toy Example with Gaussian Mixture Models}
We demonstrate the benefits of our reweighting technique through a toy Gaussian mixture model example. In Figure ~\ref{fig:toy_gmm}(a), the reference distribution is shown in blue and the biased distribution in red. The blue distribution is an equi-weighted mixture of 2 Gaussians (reference), while the red distribution is a non-uniform weighted mixture of 2 Gaussians (“biased”). The weights are 0.9 and 0.1 for the two Gaussians in the biased case. We trained a two layer multi-layer perceptron (MLP) (with tanh activations) to estimate density ratios based on 1000 samples drawn from the two distributions. We then compare the Bayes optimal and estimated density ratios in Figure ~\ref{fig:toy_gmm}(b), and observe that the estimated density ratios closely trace the ratios output by the Bayes optimal classifier.
\begin{figure}[t]
\centering
\subfigure[Biased and Reference Distributions]{\includegraphics[width=.46\textwidth]
{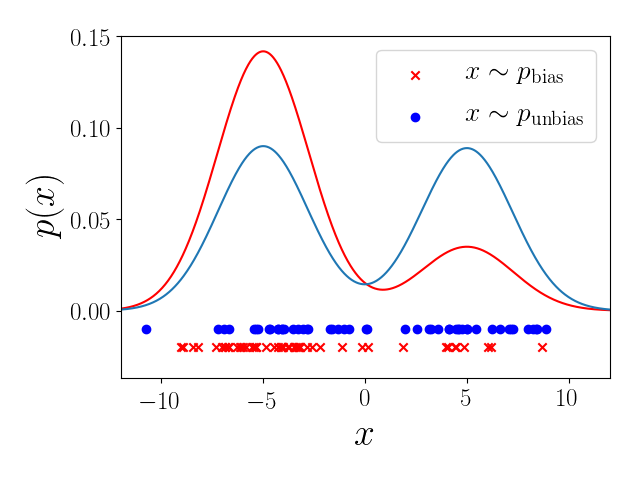}}
\subfigure[Density Ratios]{\includegraphics[width=.46\textwidth]
{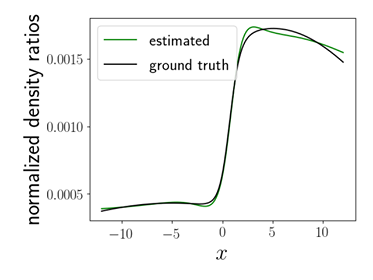}}
\caption{(a) Comparison between two biased (non-uniform weighted mixture, shown in blue) and reference (equi-weighted Gaussian mixture, shown in red). After the optimal density ratios are estimated using a two-layer MLP, we observe that the estimated density ratios are extremely similar to the ratios output by the Bayes optimal classifier, as desired.}
\label{fig:toy_gmm}
\end{figure}

\subsection{Shapes3D Dataset}
For this experiment, we used the Shapes3D dataset \citep{3dshapes18} which is comprised of 480,000 images of shapes with six underlying attributes. We chose a random attribute (floor color), restricted it to two possible instantiations (red vs. blue), and then applied Algorithm 1 in the main text for \texttt{bias}=0.9 for this setting. Training on the large biased dataset (containing excess of red floors) induces an average fairness discrepancy of 0.468 as shown in Figure ~\ref{fig:3dshapes}(a). In contrast, applying the importance-weighting correction on the large biased dataset enabled us to train models that yielded an average fairness discrepancy of 0.002 as shown in Figure ~\ref{fig:3dshapes}(b). 

\begin{figure}[t]
\centering
\subfigure[Baseline samples]{\includegraphics[width=.46\textwidth]
{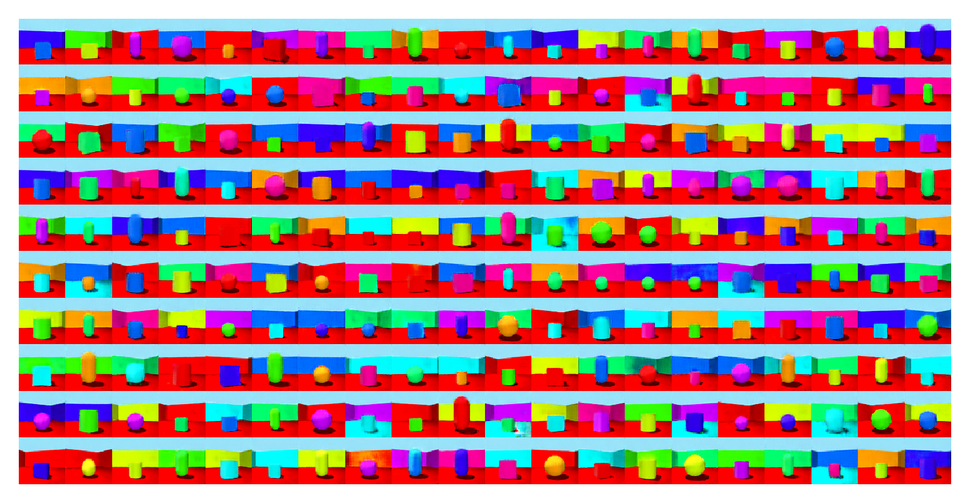}}
\subfigure[Samples after reweighting]{\includegraphics[width=.46\textwidth]
{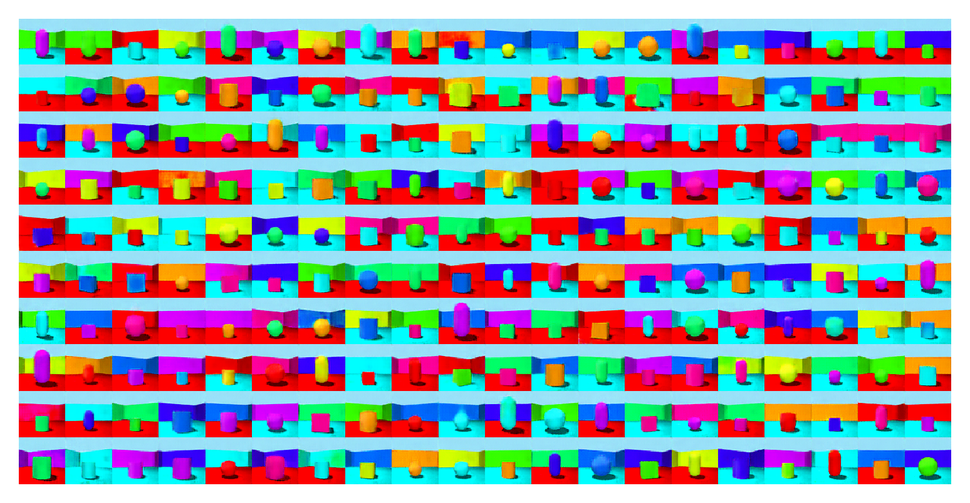}}
\caption{Results from the Shapes3D dataset. After restricting the possible floor colors to red or blue and using a biased dataset of \texttt{bias=0.9}, we find that the samples obtained after importance reweighting (b) are considerably more balanced than those without reweighting (a), as desired.}
\label{fig:3dshapes}
\end{figure}

\subsection{Downstream Classification Task}\label{sec:downstream}
We note that although it is difficult to directly compare our model to supervised baselines such as FairGAN \citep{xu2018fairgan} and FairnessGAN \citep{sattigeri2019fairness} due to the unsupervised nature of our work, we conduct further evaluations on a relevant downstream task classification task, adapted to a fairness setting.

In this task, we augment a biased dataset (165K exmaples) with a "fair" dataset (135K examples) generated by a pre-trained GAN to use for training a classifier, then evaluate the classifier's performance on a held-out dataset of true examples. We train a conditional GAN using the AC-GAN objective \citep{odena2017conditional}, where the conditioning is on an arbitrary downstream attribute of interest (e.g., we consider the “attractiveness” attribute of CelebA as in \citep{sattigeri2019fairness}). Our goal is to learn a fair classifier trained to predict the attribute of interest in a way that is fair with respect to gender, the sensitive attribute. 

As an evaluation metric, we use the \textit{demographic parity distance} ($\Delta_\textrm{dp}$), denoted as the absolute difference in demographic parity between two classifiers $f$ and $g$:
$$\Delta_\textrm{dp} = \vert f_\textrm{dp} - g_\textrm{dp} \vert$$
We consider 2 AC-GAN variants: (1) \texttt{equi-weight} trained on $\Dbias \cup \Dunbias$; and (2) \texttt{imp-weight}, which reweights the loss by the density ratio estimates. The classifier is trained on both real and generated images for both AC-GAN variants, with the labels given by the conditioned attractiveness values for the respective generations. The classifier is then asked to predict attractiveness for the CelebA test set. 

As shown in Table~\ref{table:downstream}, we find that the classifier trained on both real data and synthetic data generated by our \texttt{imp-weight} AC-GAN achieved a much lower $\Delta_{dp}$ than the \texttt{equi-weight} baseline, demonstrating that our method achieves a higher demographic parity with respect to the sensitive attribute, despite the fact that we did not explicitly use labels during training.

\begin{table}[ht]
\centering
\begin{tabular}{l|c|c|c}
\toprule
\textbf{Model} & \textbf{Accuracy} & \textbf{NLL} &\textbf{$\Delta_{dp}$} \\
Baseline classifier, no data augmentation & 79\% & 0.7964 & 0.038\\ 
\texttt{equi-weight} & \textbf{79\%} & 0.7902 & 0.032\\ 
\texttt{imp-weight} (ours) & 75\% & \textbf{0.7564} & \textbf{0.002}\\ 
\bottomrule
\end{tabular}
\caption{For the CelebA dataset, classifier accuracy, negative log-likelihood, and $\Delta_\textrm{dp}$ across \texttt{bias} = $0.9$ and \texttt{perc}=$1.0$ on the downstream classification task. Our importance-weighting method learns a fair classifier that achieves a lower $\Delta_\textrm{dp}$, as desired, albeit with a slight reduction in accuracy.}
\label{table:downstream}
\end{table}

\subsection{Single-Attribute Experiment}\label{sec:supp_80_20}
The results for the single-attribute split for \texttt{bias=0.8}
are shown in Figure~\ref{fig:80_20_results}.

\begin{figure*}[!h]
\centering
\subfigure[Samples generated via importance reweighting. Faces above orange line classified as female (55/100) while rest as male.
]{\includegraphics[width=\linewidth]
{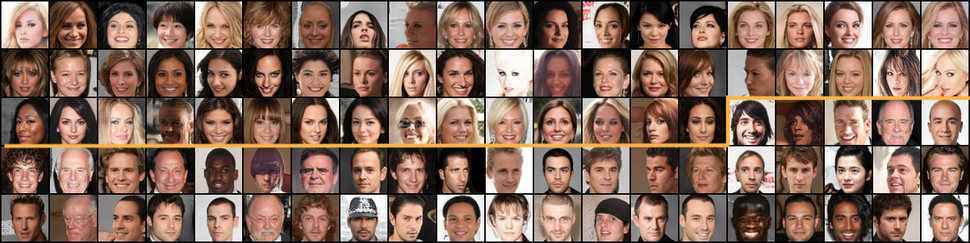}}
\subfigure[Fairness Discrepancy]{\includegraphics[width=.5\textwidth]
{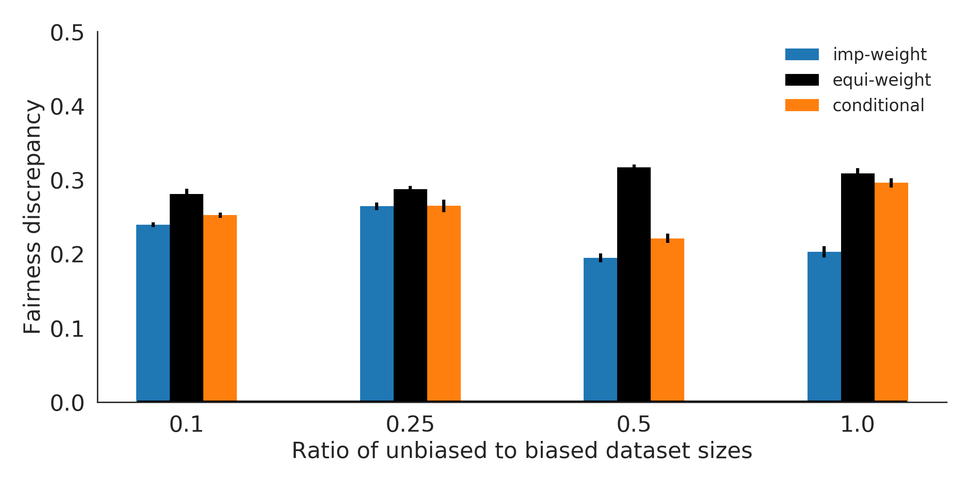}}
\subfigure[FID]{\includegraphics[width=.49\textwidth]
{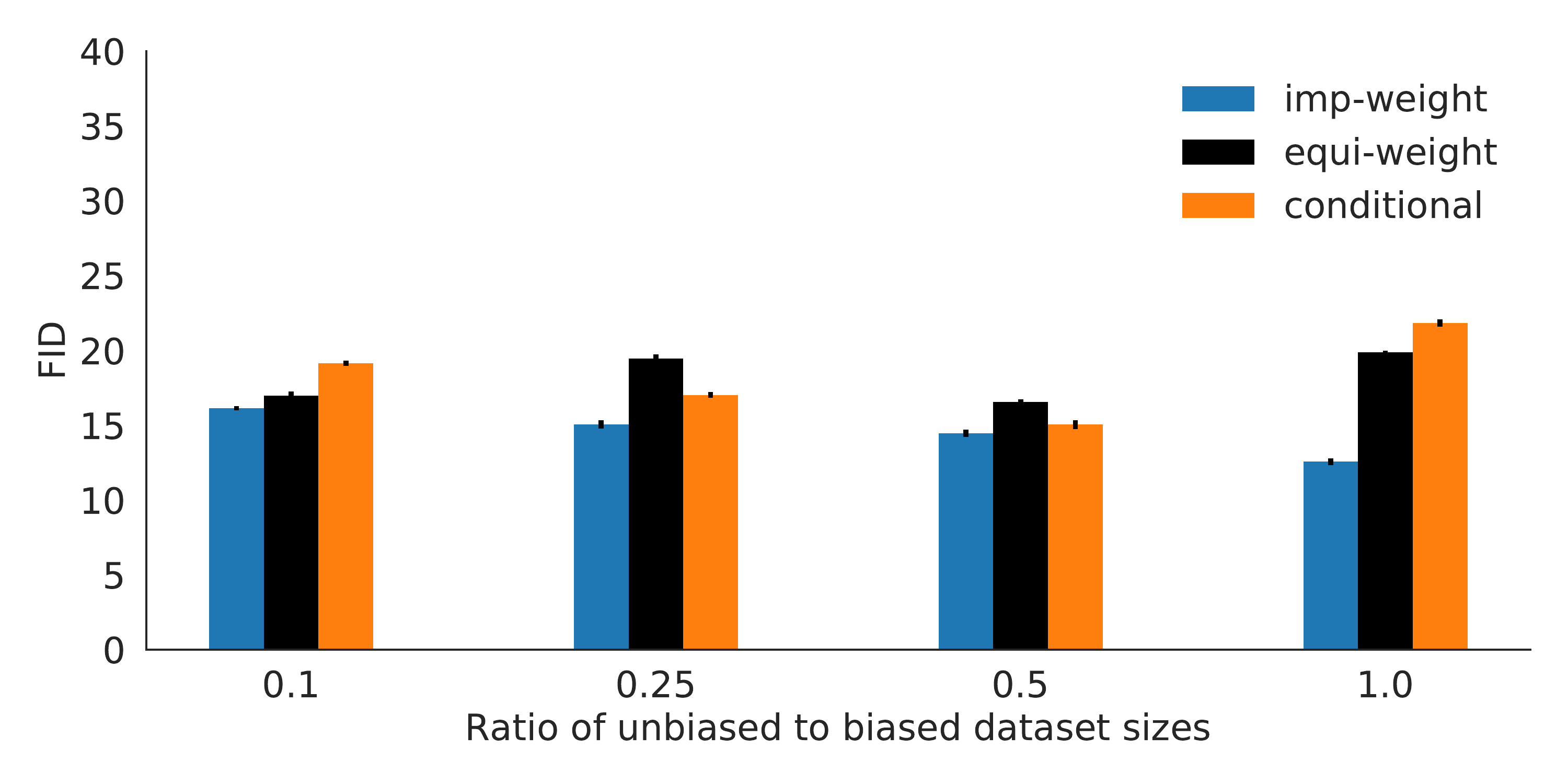}}
\caption{
Single Attribute Dataset Bias Mitigation for \texttt{bias}=0.8.
Standard error in (b) and (c) over 10 independent evaluation sets of 10,000 samples each drawn from the models. Lower fairness discrepancy and FID is better. We find that on average, \texttt{imp-weight} outperforms the \texttt{equi-weight} baseline by 23.9\% and the \texttt{conditional} baseline by 12.2\% across all reference dataset sizes for bias mitigation.
}
\label{fig:80_20_results}
\end{figure*}

\raggedbottom

\pagebreak

\section{Additional generated samples}
Additional samples for other experimental configuration are displayed in the following pages.
\begin{figure}[h!]
\centering
\subfigure[\texttt{equi-weight}]{\includegraphics[width=\textwidth]{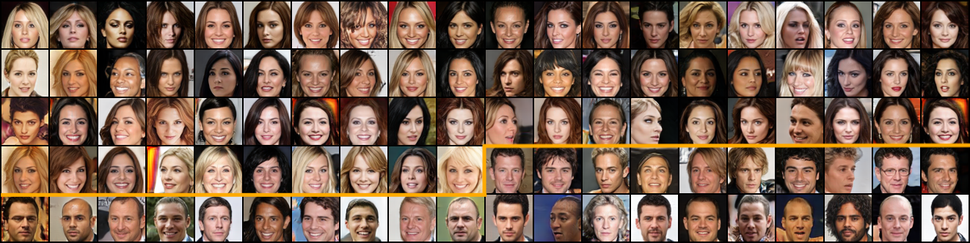}}
\subfigure[\texttt{conditional}]{\includegraphics[width=\textwidth]{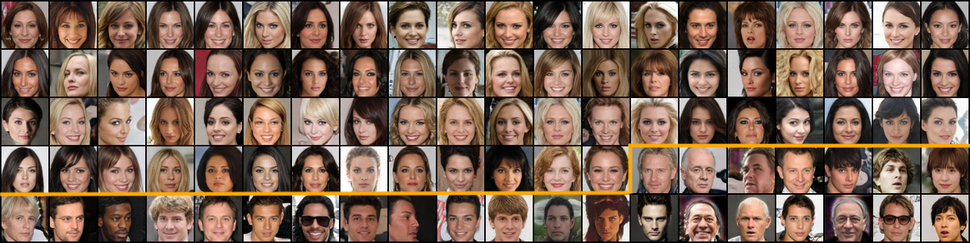}}
\subfigure[\texttt{imp-weight}]{\includegraphics[width=\textwidth]{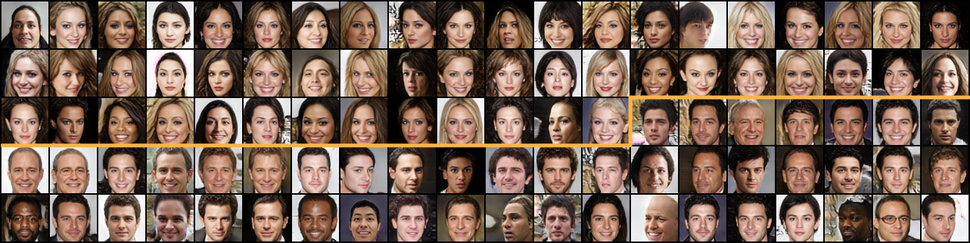}}
\caption{Additional samples of \texttt{bias=0.9}, across different methods. All samples shown are from the scenario where $|\Dunbias| = |\Dbias|$.}
\label{fig:90_10_samples}
\end{figure}

\begin{figure}[h!]
\centering
\subfigure[\texttt{equi-weight}]{\includegraphics[width=\textwidth]{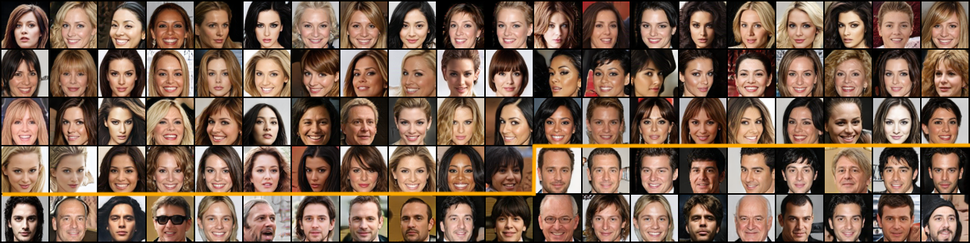}}
\subfigure[\texttt{conditional}]{\includegraphics[width=\textwidth]{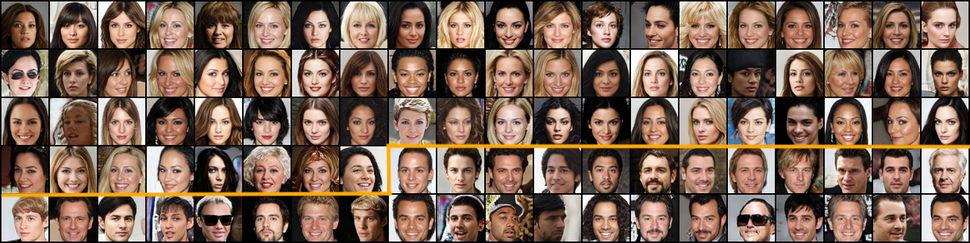}}
\subfigure[\texttt{imp-weight}]{\includegraphics[width=\textwidth]{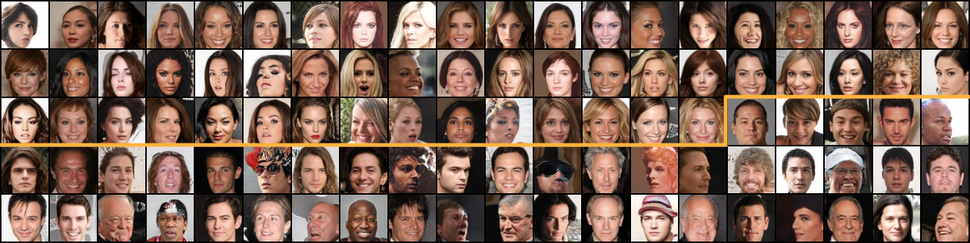}}
\caption{Additional samples of \texttt{bias=0.8}, across different methods. All samples shown are from the scenario where $|\Dunbias| = |\Dbias|$.}
\label{fig:80-20_samples}
\end{figure}

\begin{figure}[h!]
\centering
\subfigure[\texttt{equi-weight}]{\includegraphics[width=\textwidth]{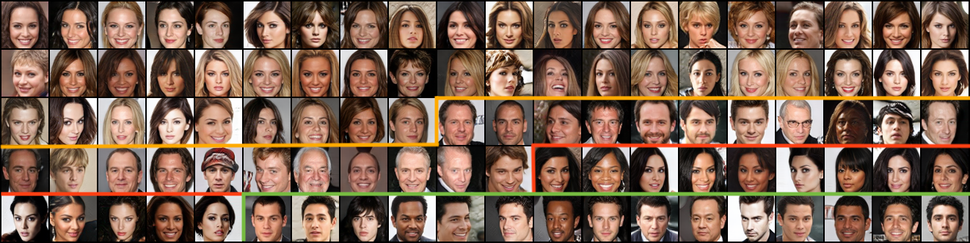}}
\subfigure[\texttt{conditional}]{\includegraphics[width=\textwidth]{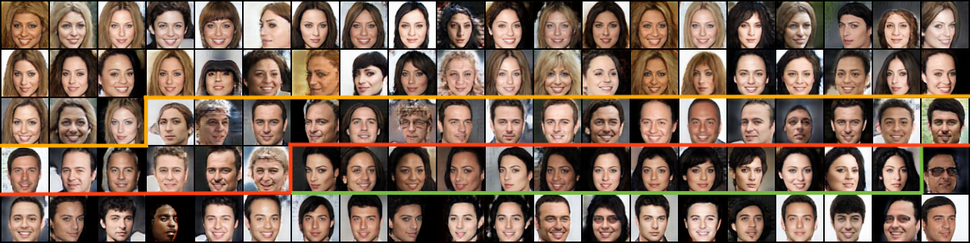}}
\subfigure[\texttt{imp-weight}]{\includegraphics[width=\textwidth]{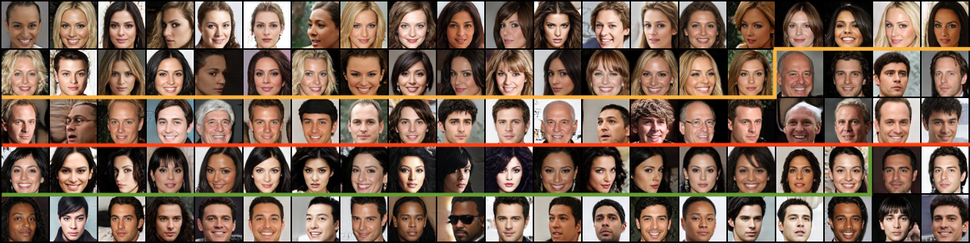}}
\caption{Additional samples of the multi-attribute experiment, across different methods. All samples shown are from the scenario where $|\Dunbias| = |\Dbias|$.}
\label{fig:multi_samples}
\end{figure}

\end{document}